\documentclass[preprint,12pt]{colt2021}

\usepackage{soul}
\usepackage{float}
\usepackage{graphicx}
\usepackage{cleveref}
\usepackage{booktabs}
\usepackage{nicefrac}
\usepackage{times}
\usepackage{enumerate}

\usepackage{footnote}
\makesavenoteenv{tabular}
\makesavenoteenv{table}

\SetCommentSty{mycommfont}

\usepackage{etoolbox}

\AfterEndEnvironment{equation}{\reseteqhint}
\AfterEndEnvironment{equation*}{\reseteqhint}
\AfterEndEnvironment{align}{\reseteqhint}
\AfterEndEnvironment{align*}{\reseteqhint}
\newcounter{hints}
\renewcommand{\thehints}{\alph{hints}}
\newcommand{\eqhint}[2][]{%
  \stepcounter{hints}%
  \if\relax\detokenize{#1}\relax\else\csxdef{hint@#1}{\thehints}\fi
  \mathrel{\overset{\textrm{(\thehints\hspace{0.01em})}}{\vphantom{\le}{#2}}}%
}
\newcommand{\reseteqhint}{\setcounter{hints}{0}}
\newcommand{\hintref}[1]{(\csuse{hint@#1})}

\makeatletter
 \let\Ginclude@graphics\@org@Ginclude@graphics 
\makeatother

\title[Optimizing Optimizers]{Optimizing Optimizers: Regret-optimal gradient descent algorithms.}

\coltauthor{
\Name{Philippe Casgrain} \Email{philippe.casgrain@math.ethz.ch}
\AND
\Name{Anastasis Kratsios} \Email{anastasis.kratsios@math.ethz.ch}
\\
\addr ETH Zürich}

\makeatletter
\DeclareOldFontCommand{\rm}{\normalfont\rmfamily}{\mathrm}
\DeclareOldFontCommand{\sf}{\normalfont\sffamily}{\mathsf}
\DeclareOldFontCommand{\tt}{\normalfont\ttfamily}{\mathtt}
\DeclareOldFontCommand{\bf}{\normalfont\bfseries}{\mathbf}
\DeclareOldFontCommand{\it}{\normalfont\itshape}{\mathit}
\DeclareOldFontCommand{\sl}{\normalfont\slshape}{\@nomath\sl}
\DeclareOldFontCommand{\sc}{\normalfont\scshape}{\@nomath\sc}
\makeatother

\DeclareMathOperator*{\Min}{Min}
\DeclareMathOperator*{\argmin}{argmin}

\newcommand{\xst}{x^{\star}}
\newcommand{\xb}{{\boldsymbol{x}}}

\newcommand{\yb}{\boldsymbol{y}}
\newcommand{\zb}{\boldsymbol{z}}

\newcommand{\delx}{\delta x}
\newcommand{\delxb}{{\boldsymbol{\delta}\xb}}
\newcommand{\fst}{f_{\star}}
\newcommand{\nuh}{\hat{\nu}}

\newcommand{\Phit}{\tilde{\Phi}}

\newcommand{\Jinf}{J^\infty}

\newcommand{\yh}{\hat{y}}

\newcommand{\gammau}{\underline{\gamma}}

\newcommand{\A}{\mathcal{A}}

\newcommand{\T}{\intercal}

\newcommand{\N}{\mathbb{N}}

\newcommand{\R}{\mathbb{R}}

\newcommand{\phit}{\tilde{\phi}}

\let\L\relax
\newcommand{\L}{\mathcal{L}}
\newcommand{\mcH}{\mathcal{H}}
\newcommand{\Lh}{\widehat{\L}}

\newcommand{\mcP}{{\mathcal{P}}}
\newcommand{\mcPh}{\hat{\mcP}}

\newcommand{\mcY}{\mathcal{Y}}

\newcommand{\mcR}{\mathcal{R}}
\newcommand{\D}{\mathcal{D}}

\newcommand{\K}{{\mathcal{K}}}

\newcommand{\X}{\mathcal{X}}

\newtheorem{assumption}{Assumption}
\newtheorem*{definition*}{Definition}

\newcommand{\NA}[1]{\error}
\newcommand{\ATTN}[1]{\error}
\newcommand{\TBD}[1]{\error}
\newcommand{\PROBLEM}[1]{\error}
\newcommand{\fTBD}[1]{\error}
\newcommand{\fPROBLEM}[1]{\error}
\newcommand{\fATTN}[1]{\error}

\begin{document}

\maketitle

\begin{abstract}%
    The need for fast and robust optimization algorithms are of critical importance in all areas of machine learning. This paper treats the task of designing optimization algorithms as an optimal control problem. Using regret as a metric for an algorithm's performance, we study the existence, uniqueness and consistency of regret-optimal algorithms. By providing first-order optimality conditions for the control problem, we show that regret-optimal algorithms must satisfy a specific structure in their dynamics which we show is equivalent to performing \emph{dual-preconditioned gradient descent} on the value function generated by its regret. Using these optimal dynamics, we provide bounds on their rates of convergence to solutions of convex optimization problems. Though closed-form optimal dynamics cannot be obtained in general, we present fast numerical methods for approximating them, generating optimization algorithms which directly optimize their long-term regret. Lastly, these are benchmarked against commonly used optimization algorithms to demonstrate their effectiveness.
\end{abstract}


\section{Introduction} \label{sec:introduction}

Let $\X=\R^d$ and consider the unconstrained minimization problem
\begin{equation} \label{eq:f-opt-problem-def}
	\Min_{x\in\X} f(x) \;,
\end{equation}
for an objective function $f:\X\rightarrow \R \cup \{ \infty \}$  which satisfies the following regularity assumptions. 
\begin{assumption} \label{ass:f-basic-assumptions}
    We assume that the optimization problem~\eqref{eq:f-opt-problem-def} is non-degenerate, in the sense that there exists a minimizer $\xst\in\argmin_{x\in\X}f(x) \subseteq \X$ for which $| f(\xst) | < \infty$.
\end{assumption}

This paper considers the problem of selecting amongst a class of optimization algorithms for those which minimize a fixed performance metric along their path. In this sense, we are concerned with a \textit{meta-optimization problem} in which we are optimizing over algorithms which in turn optimize $f$. 

We identify algorithms with the paths they take within the optimization domain $\X$. To be precise, we define an algorithm $\xb$ as the sequence of points $\xb = \{x_t\}_{t \in \N} \in \X^\N$, where $\X^\N$ is the space of sequences on $\X$. Following this notation, we introduce the set of algorithms initialized at $x \in\X$ which terminate by iteration $T\in\N$ as
\begin{equation*}
	\A_x^T := \left\{
	\xb \in \X^{\N}
	\;:\; x_0=x \;\;\text{and}\;\; \forall u \geq T,\;\; \Delta x_u = 0 \;
	\right\}
	\;,
\end{equation*} 
where we use the notation $\Delta x_t = x_{t+1} - x_t$ to represent the increments of $\xb$. We also define its asymptotic counterpart $\A_x^\infty := \{ \xb\in\X^\N : x_0 = x \}$, the set of all sequences with fixed initial point.
When necessary, we denote the union of these sets over all initial points as $\A^T := \bigcup_{x\in\X} A_x^T$. Over the course of the paper, we use the convention that bold symbols $\yb\in\X^\N$ represent algorithms and their un-bolded counterparts $y_t\in\X$ represent their value at a fixed iteration $t\in\N$.

Our focus will be on those algorithms which successfully approximate $\xst$ in the limit, i.e.\ for which $x_t \rightarrow \xst$. As previously stated, we seek the algorithms which achieve this while minimizing a measure of `performance' along the path they take. We define this measure in such a way as to represent both the speed of convergence of $\xb$ with respect to the optimization problem~\eqref{eq:f-opt-problem-def} and the `stability' of the path that the algorithm traces over a fixed horizon. Hence, we introduce the \emph{regret} of an algorithm $\xb$ as the $\mcR_T(\xb):\X^{\N} \rightarrow [0,\infty]$, given by
\begin{equation} \label{eq:regret-func-definition}
	\mcR_T(\xb) = \sum_{t=1}^T 
	f(x_t) - \fst + \phi(\Delta x_{t-1})
	\;,
\end{equation}
as our measure of an algorithm's performance. In the above definition, we assume that $\phi:\X\rightarrow [0,\infty)$ satisfies the following assumptions.
\begin{assumption}\label{ass:phi-assumptions-general}
Assume that $\phi(0)=0$, $\phi$ is lower semi-continuous and 
satisfies the growth condition that $c\|\cdot\|^p \leq \phi(\cdot)$ for some $c>0$ and $p\geq 1$, where $\| \cdot \|$ is the Euclidean norm.  
\end{assumption} 

We interpret $\mcR_T$ as measuring the performance of an algorithm based on two distinct criteria. The first component of $\mcR_T$ measures the cumulative distance to optimality through the sum of the terms $f(x_t) - \fst$, while the second measures total \emph{path energy} of $\xb$ through the sum of the terms $\phi(\Delta x_{t-1})$, which we can interpret as a generalization of the notion of its $p$-variation\footnote{ We recover the $p$-variation for $p \geq 1$ whenever $\phi(x)=\| x \|_p^p$, where $\| \cdot \|_p$ is the $p$-norm on $\X$. }. The definition~\eqref{eq:regret-func-definition} is related to the widely used notion of \emph{adversarial regret} which is the central metric of algorithmic performance in the field of online learning\footnote{We refer the reader to the text~\cite{hazan2016introduction} for a comprehensive introduction to the topic.}. The definition~\eqref{eq:regret-func-definition} is also related to the notion of \emph{regularized regret}, which is widely used in the literature on `adaptive' optimization algorithms (e.g.\ see~\cite{xiao2010dual,duchi2011adaptive,duchi2011dual}), where the main difference lies in that we regularize over the increments $\Delta x_t$, rather than the positions $x_t$. We note, however, that the definition~\eqref{eq:regret-func-definition} differs from these related notions in that it is not adversarial since $f$ remains fixed.

We are interested in algorithms which are optimal with respect to $\mcR_T$. Hence, for each $T\in\N\cup\{\infty\}$ we define the optimal control problem
\begin{equation} \label{eq:RT-control-problem-def}
	\mcP_x^T := \Min_{\xb \in \A_x^T} \mcR_T(\xb)
	\;,
\end{equation}
where we use the notation $\xb\in\mcP_x^T$ to represent an element from the set of minimizers of~\eqref{eq:RT-control-problem-def}. For $T<\infty$, elements of $\mcP_x^T$ represent algorithms with fixed starting at a point $x\in\X$, terminating at iteration $T$, which minimize the performance metric $\mcR_T$. Extending previous notation, we also introduce $\mcP^T := \bigcup_{x\in\X} \mcP_x^T$ as the set of solutions from all initial values in $\X$. For any $T\in\N\cup\{ \infty \}$, we say that an algorithm $\xb\in\mcP^T$ is \emph{regret-optimal}.

\subsection*{Summary of Main Contributions}

This paper is devoted to the study of \emph{regret-optimal} algorithms. In Section~\ref{sec:existence-and-consistency}, we characterize the existence of regret-optimal algorithms in the general non-smooth and non-convex setting, as well as their consistency across regret horizons $T\in\N\cup\{\infty\}$. In Section~\ref{sec:first-order-conditions-smth} we study regret-optimality in the setting of differentiable objectives $f$, where we derive necessary conditions on their dynamics. We furthermore show that regret-optimal algorithms admit a representation as performing \emph{dual-preconditioned gradient descent}~\citep{maddison2019dual} on their value function. Section~\ref{sec:convex-optimization} studies regret-optimality in the context of convex and differentiable objectives. This section culminates in providing a hierarchy of convergence rate bounds for regret-optimal algorithms under varying relative-smoothness and relative-convexity assumptions on the objective function $f$, which are presented in Table~\ref{tab:convergence-bound-tbl}. Lastly, Section~\ref{sec:learning-regret-optimal-algs} presents an online algorithm for the purpose of learning regret-optimality, where we apply the former to the problem of learning regret-optimal algorithm hyper-parameters on a host of toy problems.

\subsection{Related Work}

The ideas of this paper are most closely related to the various variational interpretations of optimization. In particular, we highlight~\cite{wibisono2016variational,casgrain2019latent}, which study algorithms which are critical points of an energy functional in a continuous-time setting and their connection to gradient descent algorithms with momentum. We argue that main differences between these and the present work is that we consider the former's approach \textit{ad hoc}; the variational framework in the former is chosen \textit{a posteriori} to generate momentum-like dynamics, rather than chosen \emph{a priori} to represent a concrete metric of algorithmic performance. Moreover, there are the related works of~\citet{betancourt2018symplectic,shi2019acceleration,wilson2019accelerating,francca2020dissipative} which bring the continuous analysis over to the discrete-time setting through simplectic integration methods. In contrast, our analysis deals with the discrete-time optimization problem from the very beginning without the need for supplementary discretization machinery.

This paper is also related to the body of work on control-theoretic and dynamical systems models of optimization. Of note are~\cite{lessard2016analysis,hu2017unified,muehlebach2020optimization} which present control-theoretic interpretations of the evolution dynamics of optimization algorithms. These  serve to analyze their rate of convergence to optima as well as establish various other stability properties. Though these approaches are control-theoretic, they differ from our approach since they are not concerned with \emph{optimal control}, as they do not seek controls which are optimal with respect to a fixed performance functional. Rather, they take a control as given, and study the convergence of the resulting dynamical system.

The `meta-optimization of optimizers' philosophy used in this paper has also been studied from a more applied perspective. For example~\cite{mitsos2018optimal} use this approach for automatic algorithm generation by `training' a parametric algorithm over curated examples. On the other hand~\cite{wichrowska2017learned} parametrize optimization algorithms with neural networks whose weights are learned by training on a fixed corpus of problems. Hyper-parameter tuning methods such as in~\cite{lorraine2018stochastic}, which search for optima in the set of algorithm hyper-parameters, can also be interpreted as trying to solve a finite-dimensional version of the meta-optimization problem.

\subsection{Notation, Definitions and Conventions}

For a Banach space $\mcY$, the dual space $\mcY^\ast$ represents the space of continuous linear functionals on $\mcY$. We say that $\{y_i\}_{i\in\N}\subset \mcY$ converges weakly to $y_{\infty}\in\mcY$, which we denote as $y_i\rightharpoonup y_{\infty}$, if $\ell(y_i)\rightarrow \ell(y_\infty)$ for all $\ell\in\mcY^\ast$.
%
For a convex function $g:\X\rightarrow\R$ we define its convex dual $g^\ast:\X^\ast\rightarrow\R$ as $g^{\ast}(p) = \sup_{x\in\X} \{ p(x) - f(x) \}$. For a convex and differentiable function $g$ and points $x,y\in\X$ we define the \emph{Bregman divergence} as $D_g(x,y)=g(x)-g(y)-\langle \nabla g(y) , x-y \rangle$ which is non-negative due to the convexity of $g$. For functions $g,h$ and $\mu>0$, we say that a function $g$ is \emph{$\mu$-relatively-convex} with respect to $h$ if $g-h$ is convex. Conversely, we say that $g$ is \emph{$\mu$-relatively-smooth} with respect to $h$ if $h-g$ is convex. We say that a function $g:\X\to\R$ is a positive-definite quadratic function if there exists a symmetric bi-linear form $L:\X\times\X\to\R$ such that $g(x) = L(x,x)$ and $L(x,y) > 0$ for all $0\neq x,y \in \X$.

We say a function $g:\X\rightarrow\R$ is locally Lipschitz continuous if for every $x\in\X$, there exists a compact set $K$ with non-empty interior and $L_K>0$ such that $x\in K\subset \X$ and $|g(y)-g(z)|\leq L_K \|y-z\|$ for all $y,z\in\K$. For any locally Lipschitz function $g$, we define the Clarke directional derivative at $x\in\X$ in a direction $v\in\X$ as $g^\circ (x;v)=\limsup_{y\rightarrow x,\, t\downarrow 0} \frac{f(y+tv)-f(y)}{t}$ and the generalized gradient as the set $\partial g(x) = \{ \zeta : g^\circ(x;v) \geq \langle \zeta , v \rangle \}$. If $g$ is convex then this definition coincides with its subgradient, if $g$ is differentiable then $\partial g(x)$ is a singleton containing the classical gradient, and if $x$ is a minimum of $g$ then $0\in\partial g(x)$. We point the interested reader to~\cite{ferrera2013introduction}, which covers these and other concepts of non-smooth analysis in full detail.

All proofs for theorems, lemmas and corollaries found throughout the paper are relegated to the paper's appendix. As a rule of thumb, all numbered assumptions found within the text are assumed to hold for the remainder of the paper, any other additional assumptions will be explicitly stated in the theorems, lemmas and corollaries that require them.

\section{Existence and Time-Consistency}
\label{sec:existence-and-consistency}
We begin by demonstrating the existence of regret optimal algorithms for a finite time horizon $T\in\N$, which we show are guaranteed to exist under the mild conditions put forth in assumptions~\ref{ass:f-basic-assumptions} and~\ref{ass:phi-assumptions-general}. 

\begin{theorem}\label{thm:finite-existence}
	For all $x\in\X$ and $T\in\N$, the set of minima, $\mcP_x^T$, is non-empty.
\end{theorem}

Although the control problem~\eqref{eq:RT-control-problem-def} enjoys increased analytical tractability when $T<\infty$, we are also interested in the case when $T=\infty$ since the latter admits solutions which are invariant to the iteration number, $t$. In order to 
precisely characterize the relationship between the solutions in the finite and infinite-horizon regimes in Lemma~\ref{thm:asym-consistency}, we must first introduce additional notions of regularity on the set of algorithms. For this reason, for each $\alpha \geq 0$, we introduce the set of \emph{$\alpha$-stable} algorithms,
\begin{equation*}
	\A_x^{\infty:\alpha} := \left\{
	\xb \in \X^{\N}
	\;:\; x_0=x \;\;\text{and}\;\; 
	\textstyle{ \sum_{u\in\N}}\ u^{\alpha}\|\Delta x_u\|^p <\infty 
	\right\}
	\;,
\end{equation*} 
where $p>0$ is the value found in Assumption~\ref{ass:phi-assumptions-general}.
This set can be loosely interpreted as the set of $\xb\in\A_x^\infty$ for which the increments $\| \Delta x_t \|^p$ asymptotically decay to zero at a rate $O(t^{-(1+\alpha)})$. We also note that the definition above clearly implies that $\A_x^{\infty:\alpha_1} \subset \A_x^{\infty:\alpha_0} \subset  \A_x^{\infty}$ for any $0\leq \alpha_0 \leq \alpha_1$. Following the above definition, we also define the corresponding optimization problem $\mcP_x^{\infty:\alpha}:= \argmin_{\xb \in \A_x^{\infty:\alpha}} \mcR_{\infty}(\xb)$.
In the theorem that follows, we show that 
the infinite horizon control problem is well-posed and admits solutions.
  
\begin{theorem}\label{thm:asym-existence}
Let $x\in\X$, and $\alpha \geq 0$, then $\mcP_x^{\infty:\alpha}$ is non-empty. In the case where $\alpha=0$, we have that $\mcP_x^\infty = \mcP_x^{\infty:0}$, and hence, $\mcP_x^\infty$ is also non-empty. Lastly, all solutions $\xb\in\mcP_x^{\infty:\alpha} \cup \mcP_x^{\infty}$ exhibit finite regret, so that $\mcR_\infty(\xb) < \infty$.
\end{theorem}
\begin{corollary} \label{cor:convergence-asym-nonconv}
	For any $\xb\in\mcP_x^{\infty:\alpha}$ or $\xb\in\mcP_x^{\infty}$ we have that $f(x_t) - \fst + \phi(\Delta x_t) = o(1)$. If this sequence is monotone, then we also have that $f(x_t) - \fst + \phi(\Delta x_t) = o(\frac{1}{t})$.
\end{corollary}
One important consequence of Theorem~\ref{thm:asym-existence} and Lemma~\ref{lem:mcP0-mcP-equivalence-nonempty} is that regret-optimal algorithms $\xb\in\mcP_x^{\infty:\alpha}$ or $\xb\in\mcP_x^{\infty}$
must exhibit finite regret, and hence form non-trivial solutions. Moreover, Corollary~\ref{cor:convergence-asym-nonconv} also shows that these algorithms are guarantee that $f(x_t) \to \fst$, with an asymptotic upper bound on their rate of convergence provided that they are monotone decreasing.
Another consequence is that we have the equivalence between the constrained ($\mcP_x^{\infty:0}$) and un-constrained ($\mcP_x^{\infty}$) solution sets in the $T=\infty$ regime, yielding the regularity property that $\sum_{t=0}^\infty \| \Delta x_t \|^p < \infty$ for any $\xb\in\mcP_x^{\infty}$. 
Regret-optimal algorithms also exhibit a \emph{time-consistency} property across their horizon, $T$, which we present in the next theorem.
\begin{theorem} \label{thm:asym-consistency}
Let $x \in \X$. An algorithm belongs to $\mcP_x^{\infty:\alpha}$ if and only if there exists a sequence 
$\{\xb^n\}_{n \in \N}$ such that $\xb^n\in \mcP_x^n$ and a subsequence $\{\xb^{n_k}\}_{k\in \N}$ satisfying one of the following conditions.
\begin{enumerate}
    \item If $\alpha=0$, then $\xb^{n_k} \rightharpoonup \xb^{\infty}$ in the weak topology of $\A_x^{\infty:0}$.
    \item If $\alpha>0$,  $\lim_{k\rightarrow\infty} \sum_{t=1}^{\infty}\|x_t^{n_k}-x_t^{\infty}\|^p = 0$.
\end{enumerate}
\end{theorem}
Hence, Theorem~\ref{thm:asym-consistency} shows that solutions to the infinite horizon control problem can be represented as the limit of solutions in $\mcP_x^T$, providing another avenue for computation in the $T=\infty$ regime. 
Moreover, we find that the required stability level $\alpha$ dictates the mode of convergence, where we recall that by Theorem~\ref{thm:asym-existence}, since $\mcP^{\infty:0}=\mcP^{\infty}$, the statement of Theorem~\ref{thm:asym-consistency}-\emph{2} holds for the un-constrained problem as well.

\section{Optimal Dynamics}\label{sec:first-order-conditions-smth}

A natural object of study in the context of optimal control are first-order optimality criteria for critical points of an objective. In the following section, we carry out the analysis of critical points for the regret optimization problem posed in Section~\ref{sec:introduction}. 
In order to carry out this analysis, however, we require smoothness of the control problem. As such, the focus of the remainder of the paper will be the optimization of smooth objectives.

\begin{assumption} \label{ass:f-phi-differntiable}
    Assume that the following assumptions {hold for the remainder of the paper}.
    \begin{enumerate}
        \item $f$ is everywhere differentiable.
        \item $\phi$ is Legendre convex. That is, $\phi$ is everywhere finite, strictly convex, differentiable and satisfies the \emph{super-coercivity} condition that $\lim_{\|x\|\rightarrow\infty} \| \nabla \phi(x) \| = \infty$.
    \end{enumerate}
\end{assumption}
Recall that the Legendre convexity condition in Assumption~\ref{ass:f-phi-differntiable} ensures that both $\phi$ and $\phi^\ast$ are strictly convex, differentiable and satisfy the property $\nabla \phi^\ast( \nabla \phi(x) ) = x$ for all $x\in\X$. We refer interested readers to~\cite[Section 26]{rockafellar1970convex} for more information on Legendre convex functions and their properties.

We begin our analysis of critical points by computing the Gâteaux derivative of $\mcR_T(\xb)$ over $\A_x^T$. Letting the derivative vanish, we find that the dynamics of critical points must satisfy a very specific structure which we present below.

\begin{theorem} \label{thm:gateaux-derivative}
	For $T \in \N \cup \{ \infty \}$ and $\xb\in\A_x^T$, consider the linear functional $\mcR^\prime \in(\A_x^T)^{\ast}$ defined by
	\begin{equation} \label{eq:gateaux-expression}
		\mcR_T^\prime( \xb )(\delxb) = \sum_{t=1}^{T} \langle \nabla\phi(\Delta x_{t-1}) - \nabla\phi(\Delta x_{t}) + \nabla f(x_t) \mathrel{,} \delx_{t-1} \rangle
		\;.
	\end{equation}
	If $T\in\N$, $\mcR_T^\prime$ is the Gâteaux derivative of $\mcR_T(\xb)$ over $\A_x^T$.
\end{theorem}
\begin{theorem} \label{thm:optimal-dynamics}
	For any $T\in\N\cup\{ \infty \}$, define $\mcPh_x^T \subseteq \A_x^T$ as the set of algorithms $\xb\in\A_x^T$ which satisfy the difference equation
	\begin{equation} \label{eq:optim-diff-eq}
		\nabla \phi(\Delta x_{t}) - \nabla \phi(\Delta x_{t-1}) = \nabla f(x_{t})
		\;\;\; \forall t \leq T
		\;.
	\end{equation}
	For $T<\infty$ we have the two properties that
	\begin{enumerate}
	    \renewcommand{\labelenumi}{\roman{enumi}.}
	    \item $\mcPh_x^T$ is precisely the set of critical points of $\mcR_T$, and hence $\mcPh_x^T \supseteq \mcP_x^T$.
	    \item Let $\xb\in\mcPh_x^T$, and for $h\in\N$, define the truncation $\xb_{\to h} = \{ x_{u+h} \}_{u \geq t}$. If $T\in\N$, then for any $0\leq t < T$ we have the recursive property that $\xb_{\to t} \in \mcPh_{x_t}^{T-t}$. 
	\end{enumerate}
\end{theorem} 
Theorem~\ref{thm:optimal-dynamics} therefore provides a characterization of critical points of $\mcR_T$ in terms of the difference equation~\eqref{eq:optim-diff-eq}. Readers familiar with optimal control theory can also interpret~\eqref{eq:optim-diff-eq} as the weak Pontryagin maximum principle for the control problem~\eqref{eq:RT-control-problem-def} (e.g.\ see~\cite{blot2014infinite}), where $\nabla \phi(\Delta x_t)$ fills the role of what is known as the \emph{co-state process} in optimal control and \emph{momentum} in (discrete) classical mechanics. 

Writing the explicit solution to the dynamics~\eqref{eq:optim-diff-eq}, we obtain that $\xb\in\mcPh^T$ satisfies
\begin{equation*}
    \nabla\phi(\Delta x_t) = - \sum_{u=t+1}^T \nabla f(x_u)
    \;,
\end{equation*}
which can be loosely interpreted as implying that the dynamics of $x\in\mcP^T$ are decelerating when $T\in\N$, since the number of items within the sum shrinks at each iteration. It also happens that the optimality dynamics~\eqref{eq:optim-diff-eq} admit another important interpretation in relation to the \emph{value function}. We present the results relevant to this representation below.


\begin{theorem} \label{thm:pseudo-gd-nonconvex-2}
    For $T\in\N\cup\{ \infty \}$, define the \emph{value function} 
    $x \mapsto J^T(x) := \min_{\yb\in\A_x^T} \mcR_{T}(\yb)$ over $x\in\X$. We assume one of the following.
    \begin{enumerate}
    	\item If $T<\infty$, assume that for each $0 < t \leq T$, $J^T$ is locally Lipschitz-continuous.
    	\item If $T=\infty$, assume that $\Jinf$ locally Lipschitz-continuous. 
    \end{enumerate}
    Denoting $\partial J^T(x)$ as the Clarke generalized gradient of $J^T$, for any $\xb\in\mcP^T$ and $t < T$ we have that
    \begin{equation}
    \label{eq:vector-field-dynamics-generalized}
    \begin{aligned}
		-\nabla \phi(x_{t+1} - x_t) \in \partial J^{T-t}( x_t )&
        \;,\;\;\text{ and }
		\\
		\partial J^{T-t}( x_t ) \subseteq  \partial J^{T-(t+1)}( x_{t+1} ) + \nabla f(x_{t+1} )&
		\;.
    \end{aligned}
    \end{equation}
\end{theorem}

Hence, if $T=\infty$, it is easy to see that under the assumptions of Theorem~\ref{thm:pseudo-gd-nonconvex-2}, equation~\eqref{eq:vector-field-dynamics-generalized} implies that any $\xb\in\mcP^\infty$ will satisfy the optimality dynamics of Theorem~\ref{thm:optimal-dynamics} (eq.~\eqref{eq:optim-diff-eq}). Therefore we have a result analogous to Theorem
~\ref{thm:optimal-dynamics}-\emph{1} that $\mcPh^\infty \supseteq \mcP^\infty$, and hence the dynamics of equation~\ref{eq:optim-diff-eq} are a necessary condition for optimality in the $T=\infty$ regime.

In order to better understand Theorem~\eqref{thm:pseudo-gd-nonconvex-2}, we remark that since $\phi$ is Legendre convex,  under the assumptions of Theorem~\eqref{thm:pseudo-gd-nonconvex-2} we have the representation
\begin{equation}
    x_{t+1} = x_t - \nabla \phit^\ast( \nu_t )
        \;\;\text{where}\;\;
        \nu_t \in \partial J^{T-t}(x_t)
        \;,
\end{equation}
for the iterates of $\xb\in\mcPh_x^T$, where we define $\phit(x)=\phi(-x)$. We can therefore interpret $\xb\in\mcPh^T$ as performing a variant of gradient descent on the generalized gradient of $J^{T-t}(x)$. More specifically, this variant of gradient descent happens to generalize \emph{dual-preconditioned gradient descent}~\cite[Algorithm 1.1]{maddison2019dual}. This interpretation will be particularly important in obtaining convergence bounds in Section~\ref{sec:convex-optimization}, where their connection becomes more clear.

In the case where $T=\infty$, Theorem~\ref{thm:pseudo-gd-nonconvex-2} also implies that any $\xb\in\mcP_x^\infty$ admits a map $\nu:\X \to \partial\Jinf(\X) \subseteq \X$ such that
\begin{equation} \label{eq:vector-field-dynamics}
	x_{t+1} = x_t - \nabla \phit^\ast(\nu(x_t))
	\; \text{and} \; \nu(x_t) = \nu(x_{t+1}) + \nabla f(x_{t+1})
\end{equation}
for all $t\in\N$. Hence the dynamics of such an $\xb\in\mcP^\infty$ can be uniquely represented by a vector field $\nu$ which is independent of the iteration number $t$.

\section{Convex Optimization}\label{sec:convex-optimization}

Over the course of this section, we study the regret optimization problem in the case where $f$ is convex. In particular, we will focus on the convergence of asymptotically regret-optimal algorithms $\xb\in\mcP^\infty_x$ to solutions of the optimization problem on $f$. We begin by establishing some essential convexity properties of the control problem that arise as a result.

\begin{lemma} \label{lem:convex-uniqueness}
	Assume that $f$ is convex. Then for all $T\in\N\cup\{\infty\}$, $\mcR_T(\xb)$ is a \emph{strictly convex} functional of $\A_x^T$ and hence, $\mcP_x^T$ is a non-empty singleton and $\mcP_x^T = \mcPh_x^T$.
\end{lemma}

Lemma~\ref{lem:convex-uniqueness} therefore implies that the optimality dynamics of Theorem~\ref{thm:optimal-dynamics} or~\ref{thm:pseudo-gd-nonconvex-2} are both necessary and sufficient conditions of optimality in the context of a convex control problem. Hence, any $\xb\in\A_x^T$ satisfying these dynamics is guaranteed to be the unique solution to the regret minimization control problem.

The assumption that $f$ is convex also has numerous consequences in terms of the convergence rates of regret-optimal algorithms. We study these from the perspective of the \emph{value function}
$J^T(x) := \min_{\xb\in\A_x^T} \mcR_T(\xb)$.
We note here that for each $x\in\X$, Lemma~\ref{lem:convex-uniqueness} states that there is a unique $\xb\in\mcP_x^T$ such that $J^T(x)=\mcR_T(\xb)$. As is hinted to by Lemma~\ref{thm:pseudo-gd-nonconvex-2} and the discussion that follows, we will see that this function has an important connection with the optimality dynamics of Theorem~\ref{thm:optimal-dynamics}. Before delving directly into this analysis, however, we summarize some geometric and topological facts on the value function in the convex setting.

\begin{lemma}\label{lem:convex-J-properties-1}
    Assume that $f$ is convex. Then for all $T\in\N\cup\{\infty\}$, $J^T:\X\rightarrow\R$ is a convex and differentiable function. Moreover, we have that $J^T \rightarrow \Jinf$ and $\nabla J^T \rightarrow \nabla \Jinf$ uniformly on compact sets.
\end{lemma}

\begin{lemma} \label{lem:gradJ-recursion}
    Let $T\in\N\cup\{\infty\}$ and assume that $f$ is convex. If we define $\phit(x):=\phi(-x)$ then for all $t < T$ the iterates of $\xb\in\mcP^T$ satisfy 
    \begin{equation}\label{eq:grad-DPGD-dynamics-T}
        x_{t+1} = x_t - \nabla \phit^\ast( \nabla J^T(x_t))
        \;
    \end{equation}
    as well as the recursion $\nabla J^{T-t}(x_t) = \nabla J^{T-t-1}(x_{t+1}) + \nabla f(x_{t+1})$.
\end{lemma}

Hence~\ref{lem:gradJ-recursion} shows a much clearer relationship to \emph{dual-preconditioned gradient descent} (DPGD) of~\cite[Algorithm 1.1]{maddison2019dual}. Indeed, the update rule of equation~\eqref{eq:grad-DPGD-dynamics-T} corresponds to a single step of DPGD applied for descent on the objective~$J^{T-t}$ with preconditioner $\phit$. We note that the main difference lies in that we are performing descent on the value function rather than the objective $f$. When $T=\infty$ it is easy to see that the descent is performed on $\Jinf$ at each iteration $t\in\N$.

\begin{lemma} \label{lem:convex-J-properties-2}
    Assume that $f$ is strictly convex and let $\phit(x) := \phi(-x)$. Then for all $T\in\N\cup\{\infty\}$,
	\begin{enumerate}
	\renewcommand{\labelenumi}{\roman{enumi}.}
		\item For $T\neq 0$ $J^T$ is Legendre convex and $\nabla ( J^T )^\ast = (\nabla J^T)^{-1}$.
		\item Each $J^T$ has the unique minimum $\xst$ where $\min_{x\in\X} J^T(x) = 0$.
		\item $(J^T)^\ast$ is 1-relatively-convex with respect to $\phit^\ast$. If $\phi$ is a symmetric positive-definite quadratic function then $J^T$ is also 1-relatively-smooth with respect to $\phi$.
	\end{enumerate}
\end{lemma}

Lemmas~\ref{lem:convex-J-properties-1} and~\ref{lem:convex-J-properties-2} demonstrate that the collection of value functions enjoy many regularity properties in terms of boundedness, differentiability and curvature. In particular Lemma~\ref{lem:convex-J-properties-1} shows that $J^T$ inherits both the convexity and differentiability of $f$ and $\phi$. 
Moreover, the dynamics of Lemma~\ref{lem:gradJ-recursion} along with the observation of Lemma~\ref{lem:convex-J-properties-2} that $f$ and $\Jinf$ share minima will be important for the analysis in further sections, where we study the descent of $\xb\in\mcP^\infty$ on $\Jinf$ and $f$.


Lemma~\ref{lem:convex-J-properties-2}-\emph{iii} also has the important implication that the dual of $J^T$ satisfies a relative convexity condition without any additional smoothness assumptions on $f$, which will prove crucial in the convergence analysis. In fact, in the case of a quadratic $\phi$, we show that these bounds are tightened if $f$ is also relatively smooth or convex, which we demonstrate in the following lemma.

\begin{lemma} \label{lemma:relative-smoothness-convexity-J}
    Assume that $f$ is strictly convex, and $\phi$ is a symmetric positive-definite quadratic function. Define the function $\Psi:\R_{\geq 0} \rightarrow [0,1)$ as $\Psi(x) := \frac{1}{2} \left( \sqrt{x^2 + 4 x} - x \right)$.
    \begin{enumerate}
    \renewcommand{\labelenumi}{\roman{enumi}.}
        \item If $f$ is $\lambda$-relatively smooth w.r.t. to $\phi$, then $\Jinf$ is $\Psi(\lambda)$-relatively smooth w.r.t. $\phi$.
        \item If $f$ is $\mu$-relatively convex w.r.t $\phi$, then $\Jinf$ is $\Psi(\mu)$-relatively convex w.r.t. $\phi$.
    \end{enumerate}
\end{lemma}

\subsection{Convergence on Convex Objectives}
\label{subsec:convergence-convex-objectives}

Here, we provide bounds on the rate of convergence of regret optimal algorithms in the presence of a convex objective $f$. In particular, we focus our analysis on the behaviour of \emph{asymptotically regret-optimal algorithms} $\xb\in\mcP_x^\infty$. The principal motivation behind the choice of $T=\infty$ is the time-homogeneous nature of the algorithms implied by Lemma~\ref{lem:gradJ-recursion}. We summarize the convergence rates derived over the course of this section in Table~\ref{tab:convergence-bound-tbl}.

\begin{table}[ht]
\centering
\begin{tabular}{@{}lllll@{}}
\toprule
\multicolumn{1}{c}{$f$} &
  \multicolumn{1}{c}{$\phi$} &
  Potential &
  Rate &
  Reference \\ \midrule
\multicolumn{1}{l|}{non-convex} &
  \multicolumn{1}{l|}{super-linear growth} &
  \multicolumn{1}{l|}{$f(x_t) - \fst + \phi(\Delta x_t)$} &
  \multicolumn{1}{l|}{$o(1) \,,\; o(1/t)$\,
  \footnote{See Corollary~\ref{cor:convergence-asym-nonconv} for the precise statement.}} &
  Corr.~\ref{cor:convergence-asym-nonconv}
  \\
\multicolumn{1}{l|}{strictly conv. \,/\, $\lambda$-rel.-smooth} &
  \multicolumn{1}{l|}{Legendre conv.} &
  \multicolumn{1}{l|}{$\phi^{\ast}( \sum_{u=t}^{\infty} \nabla f(x_u) )$} &
  \multicolumn{1}{l|}{$O(\nicefrac{1}{t})$} &
  Thm.~\ref{thm:convergence-basic-thm}
   \\
\multicolumn{1}{l|}{strictly conv. \,/\,  $\lambda$-rel.-smooth} &
  \multicolumn{1}{l|}{p.s.d. quadratic} &
  \multicolumn{1}{l|}{$f(x_t) - \fst + \phi(\Delta x_t)$} &
  \multicolumn{1}{l|}{$O(\nicefrac{1}{t^2})$} &
  Thm.~\ref{thm:p+1-rate-convergence}
   \\
\multicolumn{1}{l|}{$\mu$-relatively-convex} &
  \multicolumn{1}{l|}{p.s.d. quadratic} &
  \multicolumn{1}{l|}{$f(x_t) - \fst + \phi(\Delta x_t)$} &
  \multicolumn{1}{l|}{$O(e^{- \epsilon \, t})$} &
  Thm.~\ref{thm:p+1-rate-convergence}
   \\ \bottomrule
\end{tabular}
\caption{Summary of the convergence rates on convex functions for asymptotically regret-optimal algorithms $\xb\in\mcP_x^{\infty}$. We provide a reference to the exact statement with precise constants in the right-most column.}
\label{tab:convergence-bound-tbl}
\end{table}

In order to derive tighter rates of convergence for the class of asymptotically regret-optimal algorithms we leverage the connection to \emph{dual-preconditioned gradient descent}, described in the discussion following Lemma~\ref{lem:gradJ-recursion}. In what follows, we present a theorem establishing the rate of convergence of a regret-optimal algorithm with respect to the value function in the case of a convex loss. 

\begin{theorem} \label{thm:convergence-basic-thm}
	Assume that $f$ is strictly convex, $\phit$ is Legendre convex, and let $\xb\in\mcP_x^\infty$ with the associated value function $\Jinf$. Then we have the bound
	\begin{equation}
		\phit^{\ast}\left( \nabla \Jinf( x_t )\right)
		\leq \frac{ \Jinf(x_0) }{t}
		\;.
	\end{equation}
\end{theorem}

\begin{theorem} \label{thm:convergence-quadratic-thm}
    Assume that $f$ is strictly convex, $\phi$ is Legendre convex, and let $\xb\in\mcP_x^\infty$ with the associated value function $\Jinf$. Let $\lambda$ be the $\phi$-relative-smoothness constant for $\Jinf$, then we have that
	\begin{equation} \label{eq:Jinf-bound-rel-smth-quad-conv}
		\Jinf( x_t )
		\leq \frac{ \lambda \, \phi(x_0-\xst) }{t}
		\;.
	\end{equation}
	Suppose, in addition that $\Jinf$ is $\mu$-relatively-convex with respect to $\phi$. Then we have that
	\begin{equation} \label{eq:Jinf-bound-rel-conv-quad-conv}
		\Jinf(x_t) \leq \lambda \, \left( 1 - \frac{2\,\mu}{1+\mu} \right)^{t} \,  \phi(x_0-\xst)
		\;.
	\end{equation}
\end{theorem}

We note here that in the case where $f$ is just convex, we have by Lemma~\ref{lem:convex-J-properties-2}-\emph{iii} that $\lambda=1$. In the case where $f$ is either relatively convex or smooth with respect to $\phi$, we may obtain $\lambda$ and $\mu$ from Lemma~\ref{lem:interchange-relative-smoothness} which further tightens the rate of convergence.
Although the above theorems concern the rate of convergence on the value function, $\Jinf$, these also imply rates of convergence on the objective function $f$ itself, as is shown in the following theorem.

\begin{theorem} \label{thm:p+1-rate-convergence}
    Suppose that the necessary conditions for equation~\eqref{eq:Jinf-bound-rel-smth-quad-conv} hold. Then
    \begin{equation} \label{eq:t2-rate-f}
        f(x_t) - \fst + \phi(\Delta x_t) \leq \frac{ 2 \lambda \, \phi(x_0-\xst) }{t^{2}}
        \;,
    \end{equation}
    for all but finitely many $t$.
    If the necessary conditions for equation~\eqref{eq:Jinf-bound-rel-conv-quad-conv} hold, then
    \begin{equation} \label{eq:exp-rate-f}
        f(x_t) - \fst + \phi(\Delta x_t) \leq 
         \lambda \, \phi(x_0 - \xst) \, \left( 1 - \frac{2\,\gamma}{1+\gamma} \right)^{t+1}
        \;
    \end{equation}
    for all but finitely many $t$.
\end{theorem}

The bounds we provide over this section improve upon the general non-convex bound in Corollary~\ref{cor:convergence-asym-nonconv}. Moreover, we show that with additional assumptions on the relative smoothness and relative convexity of the objectives, we can further tighten these rates. In contrast to Corollary~\ref{cor:convergence-asym-nonconv}, we provide exact constants on the rate of convergence. We also point out that the bounds in Table~\ref{tab:convergence-bound-tbl} happen to co-incide exactly with known lower bounds for the rate of convergence of gradient-based optimization algorithms as shown in~\cite{nesterov2003introductory}. 
In particular, the $O(t^{-2})$ rate of equation~\eqref{eq:t2-rate-f} implies that asymptotically-regret-optimal algorithms achieve rates of convergence on relatively-smooth objectives in the same class as the Nesterov accelerated gradient algorithm of~\cite{nesterov1983method} and its variants.

\begin{remark}
	Although the analysis over the course of this section is applied for deriving rates of convergence for algorithms $\xb\in\mcP^\infty$, very similar results can be derived for $\xb\in\mcP^T$ with $T<\infty$ using the same techniques.
\end{remark}

\section{Learning Regret-Optimal Algorithms}\label{sec:learning-regret-optimal-algs}

For an asymptotically regret-optimal algorithm $\xb\in\mcP_x^\infty$, we turn to the problem of computing its dynamics so that it can be applied to tangible optimization problems. We approach the problem by taking the perspective of the optimal dynamics of Theorem~\ref{thm:optimal-dynamics}. Although the optimal dynamics~\eqref{eq:optim-diff-eq} exhibit closed form solutions in the case of descent on a quadratic objective, as is presented in Appendix~\ref{sec-apx:solution-linear-gradients}, there are no such solutions available for general $f$. Hence, we turn to numerical methods for the approximation of the optimal dynamics presented in equation~\eqref{eq:optim-diff-eq}.

Let us assume a vector field $\nuh^{\theta}:\X\rightarrow\X$ parametrized by $\theta\in\Theta$, for some vector space $\Theta$. Our approach will be to estimate the $\theta$ such $\nuh^\theta$ approximately satisfies the necessary optimality criteria of Theorem~\ref{thm:pseudo-gd-nonconvex-2}. Hence, we define the loss function
\begin{align} \label{eq:loss-function-full}
    &\L(\theta;x) =
    \left\|
        \nuh^{\theta}(x) - \nuh^{\theta}( \yh^{\theta}(x) ) - \nabla f( \yh^{\theta}(x) )
    \right\|^2
    \\
    &\text{where }\;
    \yh^{\theta}(x) = x - \nabla \phit^{\ast}( \nuh^{\theta}(x) )
    \label{eq:vf-update-rule}
    \;,
\end{align}
which is obtained by taking the squared norm of the difference between both sides of the second line of equation~\eqref{eq:vector-field-dynamics}. Just as Q-learning serves as a method for approximating a function which satisfies the Bellman equation, one can interpret minimizing~\eqref{eq:loss-function-full} as identifying a vector field which approximately satisfies the Pontryagin Maximum Principle. The loss function~\eqref{eq:loss-function-full} also admits a more direct interpretation in connection with the Gâteaux derivative of the regret, which is presented in the following lemma.
\begin{lemma} \label{lem:loss-gateaux-deriv-connection}
    Let $\xb^\theta\in\A^\infty$ be the algorithm induced by the vector field $\nuh^{\theta}$ according to equation~\eqref{eq:vf-update-rule} and let $\mcR_T^\prime$ be the Gâteaux derivative of $\mcR_T$.  Then we have that
    \begin{equation} \label{eq:loss-regret-deriv}
        \sum_{t=1}^T \L(\theta;x_{t-1}^\theta)
        =
        \left\| \mcR_T^\prime(\xb^\theta) \right\|^2_{2,\ast}
        \;,
    \end{equation}
    where $\| \cdot \|_{{2,\ast}}$ is the dual norm with respect to the norm on $\A^\infty$ defined by $\| \xb \|^2_{2} = \sum_{t\in\N} \| x_t \|^2$.
\end{lemma}
Hence, we may loosely interpret minimizing $\bar{\L}_\infty(\theta) = \sum_{t\in\N} \L(\theta;x_{t-1}^\theta)$ as minimizing the norm of the Gâteaux derivative of $\mcR_{\infty}$, and drawing $\xb^{\theta}$ closer to a critical point of $\mcR_{\infty}$. 

In order to minimize $\bar{\L}_\infty$, we rely on online gradient descent or an equivalent algorithm to minimize the online loss function
$\L(\theta_t,x_{t-1})$ at each iteration. We note that each evaluation of $\L(\theta_t,x_t)$ requires an evaluation of the gradient of $f$. In order to reduce the number of gradient evaluations, we consider the following approximation to $\L$
\begin{equation}
    \Lh_t(\theta) =
    \left\|
    \nuh^{\theta}(x_t) - \nuh^{\theta}(\yh^{\theta}(x_t) ) - \nabla f( \yh^{\theta_{t-1}}(x_t) )
    \right\|^2
    \;,
\end{equation}
in which we `freeze' $\theta$ in the expression $\nabla f( y^{\theta} )$ within the above equation. We summarize these ideas in the Algorithm~\ref{alg:online-regret-optimization}, which aims at minimizing $\bar{\L}_\infty(\theta)$ in an online manner.

\begin{algorithm2e}[ht]
\linespread{1.15}\selectfont
\SetKwInOut{Input}{input}
\SetKwInOut{Output}{output}
\SetKwFunction{Optimize}{Optimize}
\SetKwFunction{Collect}{Collect}
\SetAlgoLined
\Input{
$x_0\in\X$, $T\in\N$, $\nabla \phi^{\ast}$, 
$\theta_0$, $\nuh^{\theta}$
}
\For{$t\;\leftarrow\;0\;$ \KwTo $\;T-1$}{
  $\yh_t \leftarrow x_t - \nabla\phi^{\ast}( \, \nuh^{\theta_{t}} (x_t) \, )$
  \tcp{Compute 'test' point for gradient evaluation.}
  $\theta_{t+1} \; \leftarrow \;$ \Optimize{$\;\Lh_t( \theta )\;$}
  \tcp{Update $\theta$ estimate.}
  $x_{t+1} \; \leftarrow \; x_t - 
  \nabla \phi^{\ast}( \, \nuh^{\theta_{t+1}} (x_t) \, ) $
  \tcp{Step forward with new estimate.}
 }
\Return{$x_T$}
\caption{Online Regret Meta-Optimization}
\label{alg:online-regret-optimization}
\end{algorithm2e}

In Algorithm~\ref{alg:online-regret-optimization}, we assume that the optimization routine within the inner loop can be computed quickly and with a small memory footprint. To achieve this, one can ensure that $\dim(\Theta) \ll \dim(\X)$ so that gradients of $\nuh^{\theta}$ can be computed cheaply. Moreover, at each iteration, the optimization routine does not need to fully optimize $\L_t$ and can be replaced by a single iteration of gradient descent. As long as the inner optimization loop can be ensured to be fast, Algorithm~\ref{alg:online-regret-optimization} can serve as a viable option to adaptive optimization algorithms and can be applied to a wide range of optimization problems, regardless of the size of $\dim(\X)$.

\subsection{Auto-Tuning Gradient Descent with Momentum}

We present a simple numerical example for Algorithm~\ref{alg:online-regret-optimization}, which we test on some basic optimization objectives. We choose a two-parameter model of the form
\begin{equation*}
    \nuh^{\theta}_t = \alpha \, \nabla f(\, y_{t-1} \,) + \beta \, \nuh^{\theta}_{t-1}
    \;,
\end{equation*}
where $\theta=\{\alpha,\beta\}$ and we assume that $\alpha,\beta>0$ and where $\nabla f(\, y_{t-1} \,)$ is the last gradient that has been evaluated. We can interpret this model as a parametrized version of gradient descent with momentum, where $\alpha,\beta$ control the weights on the gradient and momentum, respectively. We use a single step of gradient descent as the optimization routine for the algorithm.

We compare the performance of the resulting online algorithm with  gradient descent and Nesterov accelerated gradient descent, each with fixed hyperparameters. We use Algorithm~\ref{alg:online-regret-optimization} with $\phi(x) = \frac{\gamma^{-1}}{2} \| x \|^2$, where $\gamma>0$ is the learning rate. We set the learning rates to be the same value for all algorithms that are compared. In order to optimize in the inner loop of Algorithm~\ref{alg:online-regret-optimization}, we run 10 steps of gradient descent with a learning rate of $10^{-4}$.
We apply each algorithm on two types of examples, first on a rescaled Rosenbrock function\footnote{Let $f_r(x,y)$ be the Rosenbrock function on $\R^2$ \citep{rosenbrock1960automatic}. We define the rescaled Rosenbrock function as $(x,y) \mapsto 0.1 \,\sqrt{f_r(0.5 \,x, 4.5\, y)}$ which has the property that it is relatively smooth w.r.t. $\phi$.} on $\R^2$, with results displayed in Figure~\ref{fig:2D-optimization} and second on a randomly generated symmetric positive-definite quadratic objective on $\R^{2^{12}}$ with results displayed in Figure~\ref{fig:20kD-optimization}. These examples serve as a proof of concept and show that Algorithm~\ref{alg:online-regret-optimization} works comparatively well to two other well-known optimization algorithms on toy problems. 


\begin{figure}[ht]
\centering
\begin{minipage}[b]{0.32\textwidth}
    \includegraphics[width=0.99\textwidth]{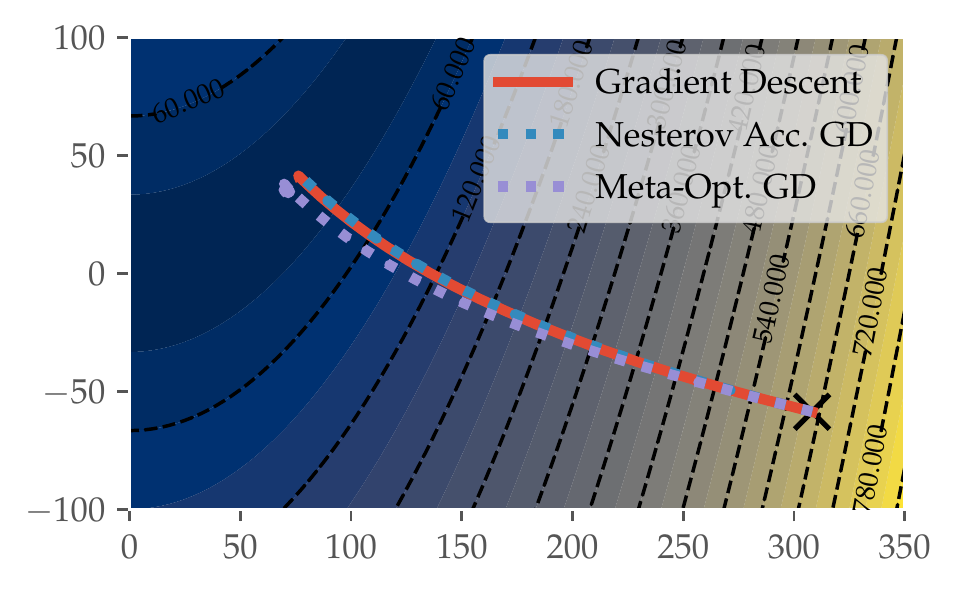}
\end{minipage}
\begin{minipage}[b]{0.32\textwidth}
    \includegraphics[width=0.99\textwidth]{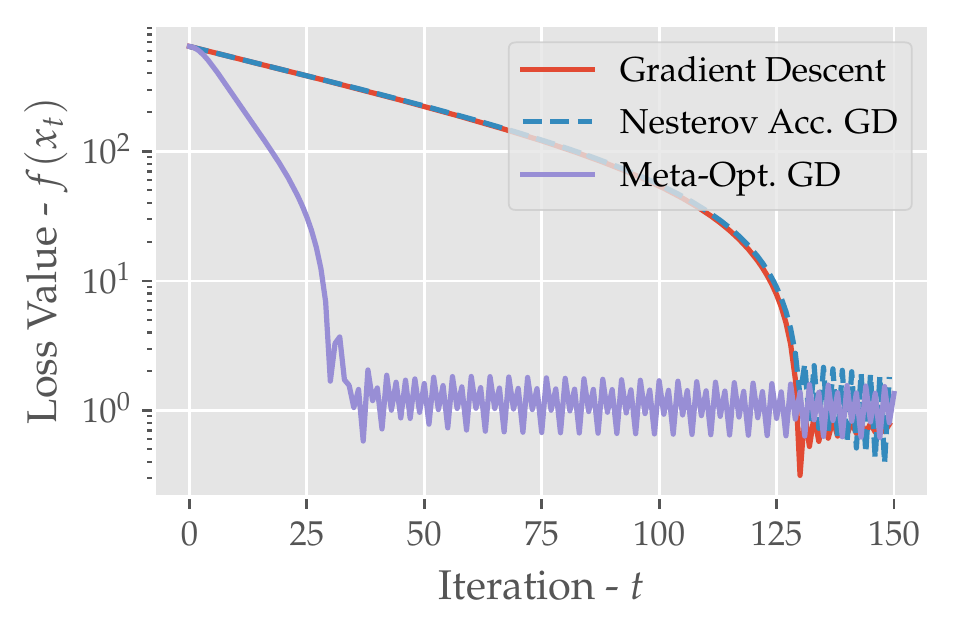}
\end{minipage}
\begin{minipage}[b]{0.32\textwidth}
    \includegraphics[width=0.99\textwidth]{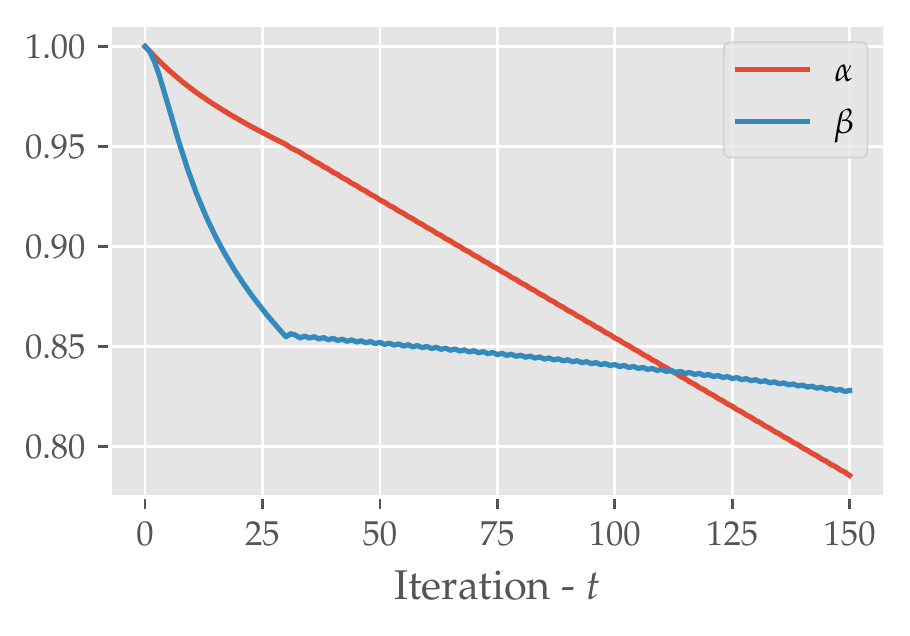}
\end{minipage}
\caption{ \textbf{Rescaled Rosenbrock Objective.} We include (left to right) a contour plot with the paths of each algorithm in $\X$, a plot of the loss function value over each iteration and the evolution of the hyperparameters, $\theta=\{\alpha,\beta\}$, of Algorithm~\ref{alg:online-regret-optimization}.}
\label{fig:2D-optimization}
\end{figure}

\begin{figure}[ht]
\centering
\begin{minipage}[b]{0.32\textwidth}
    \includegraphics[width=0.99\textwidth]{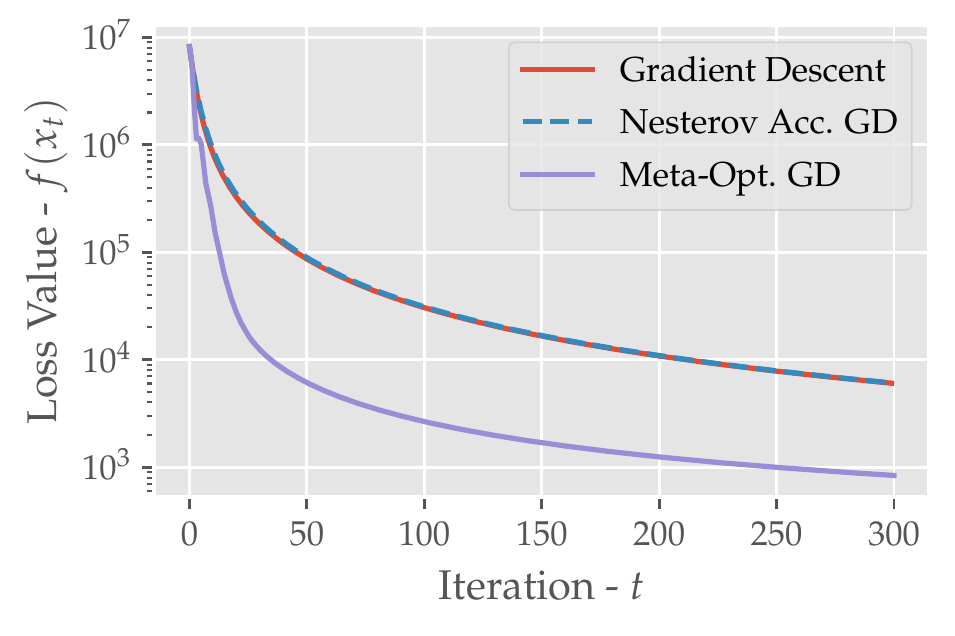}
\end{minipage}
\begin{minipage}[b]{0.32\textwidth}
    \includegraphics[width=0.99\textwidth]{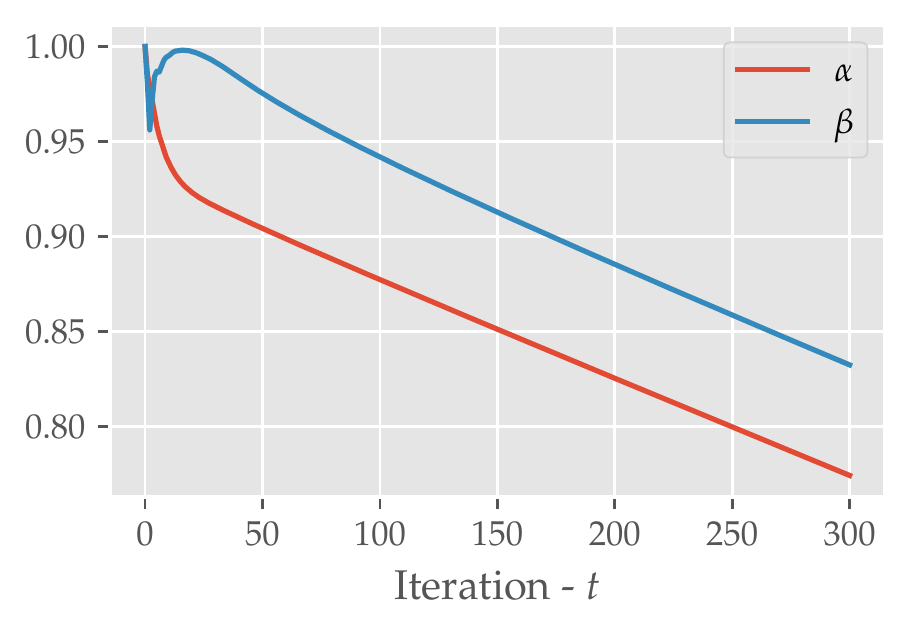}
\end{minipage}
\caption{ \textbf{$\mathbf{2^{12}}$-Dimensional Quadratic Objective.} We include (left to right) a plot of the loss value over each iteration and the evolution of the algorithm hyperparameters, $\theta$. }\label{fig:20kD-optimization}
\end{figure}

\section{Discussion}

Over the course of this paper, we characterize the existence and properties of regret optimal algorithms in a wide range of common optimization settings. One shortcoming of our approach, however, is that we do not restrict the measurability of algorithms in the minimization of regret. In light of this fact, it is interesting that we recover in Table~\ref{tab:convergence-bound-tbl} bounds that look quite similar to optimal convergence bounds for gradient-based optimization algorithms, in particular the $O(t^{-2})$ bound that is known to hold for `accelerated' algorithms. A more in depth analysis of these rates and comparison to known lower bounds would also be very interesting.

This paper presents new perspectives on optimization which deserve to be further explored. An interesting potential avenue of research would be the extension of this framework towards stochastic optimization. Another direction would be to see how commonly used optimization algorithms fall within this framework, and to determine whether they satisfy regret optimality in an exact or approximate sense.

\bibliography{bibfiles/optimopt_refs}

\newpage
\appendix

\section{Regret-Optimal Dynamics on Quadratic Objectives} \label{sec-apx:solution-linear-gradients}

\begin{lemma}[Descent on Quadratic Functions]
Assume that there exist $A,C\in S^d_{++}$ and $b\in\R^d$ such that
\begin{equation*}
  \nabla f(x) = A \, ( x - b )
  \;\;\text{ and }\;\;
  \phi(z) = \frac{1}{2} z^{\T} C z
  \;.
\end{equation*}
If there exists a matrix $\Phit \in S^d_{++}$ such that 
\begin{equation*}
    \Phit^{-1} = C^{-1} + (A + \Phit)^{-1}
    \;,
\end{equation*}
then we have that
\begin{equation*}
  \Delta x_{t} = - C^{-1}  \Phi_1 \left( x_t - b \right)
  \;,
\end{equation*}
for all $t\in\N$.
\end{lemma}
\begin{proof}
    First note that under the assumptions of the theorem, we have that there exists a constant $c$ such that $f(x) = c + \frac{1}{2}(x-b)^\T A (x-b)$, demonstrating that $f$ convex function. Moreover, if $\xb\in\mcP^\infty$, we also have by Lemma~\ref{lem:bregman-recursion} that
    \begin{equation} \label{eq:proof-linear-dynamics-1}
        (\Jinf)^\ast(p) = \phi^\ast(p) + (f + \Jinf)^\ast(p)
        \;,
    \end{equation}
    where by a simple computation, we have that $\phi^\ast(p) = \frac{1}{2} p^\T C^{-1} p$. Moreover, we note that without loss of generality, we may consider the case where $b=0$ since we may simply use the domain transformation $x \mapsto x+b$.

    Now, assume that there exists $\Phit \in S^d_{++}$ such that $\nabla \Jinf(x) = \Phit x$. Differentiating equation~\eqref{eq:proof-linear-dynamics-1} and noting that the linearity of the gradients of $\Jinf$ and $f$ imply that $\nabla (f + \Jinf)^\ast (p) = (A + \Phit)^{-1} p$, we get that
    \begin{equation*}
        \Phit^{-1} p = C^{-1} p + (A + \Phit)^{-1} p
        \;,
    \end{equation*}
    which must hold for all $p\in\X$. Hence, if there exists $\Phi\in S^d_{++}$ such that
    \begin{equation*}
        \Phit^{-1} = C^{-1} + (A + \Phit)^{-1}
        \;,
    \end{equation*}
    noting that $\Delta x_t = - C^{-1} \nabla \Jinf(x)$, we find that the statement of the theorem must hold.
\end{proof}

\newpage

\section{Auxiliary Results} 
\label{sec:auxiliary-results}

\subsection{Dynamic Programming Principles}

\begin{lemma}[Dynamic Programming Principle, $T<\infty$]
	\label{lem:dynamic-programming-principle-fin}
	Suppose that Assumptions~\ref{ass:f-basic-assumptions} and~\ref{ass:phi-assumptions-general} hold. If for any $T\in\N$ we define the value function~$J^T(x)=\min_{\xb\in\A_x^T} \mcR_T(\xb)$, then for each $t,T\in\N$, the iterates of $\xb$ satisfy
	\begin{equation*} \label{eq:statement-bellman-principle-fin}
		J^T(x_t) = \min_{y\in\X} \{ \phi(y-x_t) + f(y) - \fst + J^{T-1}(y) \}
		\;,
	\end{equation*}
	where $x_{t+1}\in\argmin_{y\in\X} \{ \phi(y-x_t) + f(y) - \fst + J^{T-1}(y) \}$. Moreover, for each $0<h<T$, defining the shifted sequence $\xb_{\to h}$ such that $\xb_{\to h}=\{ y_{t+h} \}_{t=0}^\infty$, we have that $ \xb_{\to h}\in\mcP^{T-h}_{x_h}$.
\end{lemma}
\begin{proof}
	We assume without loss of generality that $\fst=0$. We first note that by Theorem~\ref{thm:finite-existence}, we have that $\mcP_x^T\neq \emptyset$ for all $T\in\N$ and $x\in\X$. Moreover, by the definition of $\mcP_x^T$, we also have that $J^{T}(x) = \mcR_T(\xb) < \infty$ for any $\xb\in\mcP_x^T$.

	Recursively expanding $\mcR_{T}$, we find that for any $\yb\in\A^T$ such that $\mcR_{T}(\yb)<\infty$ and $0\leq h \leq T$, we have that
	\begin{equation*}
		\mcR_{T}(\yb) = \mcR_{h}(\yb) + \mcR_{T-h}(\yb_{\to h})
		\;,
	\end{equation*}
	where $\yb_{\to h}$ is the shifted sequence $\yb_{\to h}=\{ y_{t+h} \}_{t=0}^h$. The above recursion also includes the special case that
	\begin{equation} \label{eq:proof-bellman-principle-fin-1}
		\mcR_{T}(\yb) = \phi(y_1-y_0) + f(y_1) + \mcR_{T-1}(\yb_{\to 1})
		\;.
	\end{equation}

	Now we show that $\xb\in\mcP^{T}_x \; \Rightarrow \; \xb_{\to 1}\in\mcP^{T-1}_{x_1} $. Assume the converse that $\xb_{\to 1}\notin\mcP^{T-1}_{x_1}$ (i.e.\ that $\mcR_{T-1}(\xb_{\to 1}) > J^{T-1}(x_1)$ ). Applying \eqref{eq:proof-bellman-principle-fin-1} we have that
	\begin{align*}
		J^{T}(x) &= \mcR_T(\xb)
		\\ &= \phi(x_1 - x_0) + f(x_1) + \mcR_{T-1}(\xb_{\to 1})
		\\ &> \phi(x_1 - x_0) + f(x_1) + J^{T-1}(x_1)
		\;.
	\end{align*}
	But, defining a control $\yb\in\A_x^{T}$ such that $y_0 = x_1$ and $\yb_{\to 1} \in \mcP_x^{T-1}$, we have that by the definition of $\xb\in\mcP_x^{T-1}$,
	\begin{align*}
		J^{T}(x) &\leq \mcR_T(\yb)
		\\ &= \phi(x_1 - x_0) + f(x_1) + \mcR_{T-1}(\yb_{\to 1})
		\\ &= \phi(x_1 - x_0) + f(x_1) + J^{T-1}(x_1)
		\;,
	\end{align*}
	which is a contradiction. Hence, we find that $\xb_{\to 1}\in\mcP^{T-1}_{x_1}$ and
	\begin{equation} \label{eq:proof-bellman-principle-fin-2}
		J^{T}(x_0) = \phi(x_1 - x_0) + f(x_1) + J^{T-1}(x_1)
		\;,
	\end{equation}
	for all $\xb\in\mcP_x^{T}$ and $x\in\X$. By induction on $h$, we also note that $\xb_{\to 1}\in\mcP^{T-1}_{x_1}$ also implies that $\xb_{\to h}\in\mcP^{T-h}_{x_h}$ for $0<h<T$.

	Now suppose that for $\xb\in\mcP^T_x$, we have that $x_1 \notin B_x = \argmin_{y\in\X} \{ \phi(y-x_0) + f(y) + J^{T-1}(y) \}$. Then by~\eqref{eq:proof-bellman-principle-fin-2}, there exists $y\in\X$ such that
	\begin{equation*}
		\phi(y-x_0) + f(y) + J^{T-1}(y) <  \phi(x_1 - x_0) + f(x_1) + J^{T-1}(x_1) = J^T(x_0)
		\;.
	\end{equation*}
	Conversely, defining $\yb\in\A_x^T$ such that $\yb_{\to 1} \in \mcP_{y}^{T-1}$, we have that by the definition of $J^\infty$ that
	\begin{align*}
		J^T(x) &\leq \mcR(\yb)
		\\ &= \phi(y - x) + f(y) + \mcR_{T-1}(\yb_{\to 1})
		\\ &= \phi(y - x) + f(y) + J^{T-1}(y)
		\;,
	\end{align*}
	which is again a contradiction, and hence $x_1 \in\argmin_{y\in\X} \{ \phi(y-x_0) + f(y) + J^{T-1}(y) \}$. Combining this result with~\eqref{eq:proof-bellman-principle-fin-2}, and noting that $t$ is arbitrary we therefore conclude the proof.
\end{proof}

\begin{lemma}[Dynamic Programming Principle, $T=\infty$]
	\label{lem:dynamic-programming-principle-asym}
	Suppose that $\mcP_x^\infty$ is non-empty $\forall x\in\X$ and that Assumptions~\ref{ass:f-basic-assumptions} and~\ref{ass:phi-assumptions-general} hold. If we define the value function~$J^\infty(x)=\min_{\xb\in\A_x^\infty} \mcR_{\infty}(\xb)$, then for each $t$, the iterates of $\xb$ satisfy
	\begin{equation*} \label{eq:statement-bellman-principle-asym}
		J^\infty(x_t) = \min_{y\in\X} \{ \phi(y-x_t) + f(y) - \fst + J^\infty(y) \}
		\;,
	\end{equation*}
	where $x_{t+1}\in\argmin_{y\in\X} \{ \phi(y-x_t) + f(y) - \fst + J^\infty(y) \}$. Moreover, for each $T\in\N$, defining the shifted sequence $\xb_{\to h}$ such that $\xb_{\to h}=\{ y_{t+h} \}_{t=0}^\infty$, we have that $ \xb_{\to h}\in\mcP^\infty_{x_h}$.
\end{lemma}
\begin{proof}
	We assume without loss of generality that $\fst=0$. We first note that by Theorem~\ref{thm:asym-existence}, we have that $0\leq J^\infty(x) < \infty$ for all $x\in\X$. Moreover, by the definition of $\mcP_x^\infty$, we also have that $J^{\infty}(x) = \mcR_\infty(\xb)$ for any $\xb\in\mcP_x^\infty$. The remainder of the proof resembles closely the proof of Lemma~\ref{lem:dynamic-programming-principle-fin}.

	Recursively expanding $\mcR_{\infty}$, we find that for any $\yb\in\A^\infty$ such that $\mcR_{\infty}(\yb)<\infty$ and $h>0$, we have that
	\begin{equation*}
		\mcR_{\infty}(\yb) = \mcR_{h}(\yb) + \mcR_{\infty}(\yb_{\to h})
		\;,
	\end{equation*}
	where $\yb_{\to h}$ is the shifted sequence $\yb_{\to h}=\{ y_{t+T} \}_{t=0}^h$. The above recursion also includes the special case that
	\begin{equation} \label{eq:proof-bellman-principle-1}
		\mcR_{\infty}(\yb) = \phi(y_1-y_0) + f(y_1) + \mcR_{\infty}(\yb_{\to 1})
		\;.
	\end{equation}

	Now we show that $\xb\in\mcP^\infty_x \; \Rightarrow \; \xb_{\to 1}\in\mcP^\infty_{x_1} $. Assume the converse that $\xb_{\to 1}\notin\mcP^\infty_{x_1}$ (i.e.\ that $\mcR_\infty(\xb_{\to 1}) > \Jinf(x_1)$ ). Applying \eqref{eq:proof-bellman-principle-1} we have that
	\begin{align*}
		\Jinf(x) &= \mcR_\infty(\xb)
		\\ &= \phi(x_1 - x_0) + f(x_1) + \mcR_{\infty}(\xb_{\to 1})
		\\ &> \phi(x_1 - x_0) + f(x_1) + \Jinf(x_1)
		\;.
	\end{align*}
	But, defining a control $\yb\in\A_x^\infty$ such that $y_0 = x_1$ and $\yb_{\to 1} \in \mcP_x^\infty$, we have that by the definition of $\xb\in\mcP_x^\infty$,
	\begin{align*}
		\Jinf(x) &\leq \mcR_\infty(\yb)
		\\ &= \phi(x_1 - x_0) + f(x_1) + \mcR_{\infty}(\yb_{\to 1})
		\\ &= \phi(x_1 - x_0) + f(x_1) + \Jinf(x_1)
		\;,
	\end{align*}
	which is a contradiction. Hence, we find that $\xb_{\to 1}\in\mcP^\infty_{x_1}$ and
	\begin{equation} \label{eq:proof-bellman-principle-2}
		\Jinf(x_0) = \phi(x_1 - x_0) + f(x_1) + \Jinf(x_1)
		\;,
	\end{equation}
	for all $\xb\in\mcP_x^\infty$ and $x\in\X$. By induction on $h$, we also note that $\xb_{\to 1}\in\mcP^\infty_{x_1}$ also implies that $\xb_{\to h}\in\mcP^\infty_{x_h}$ for any $h\in\N$.

	Now suppose that for $\xb\in\mcP^\infty_x$, we have that $x_1 \notin B_x = \argmin_{y\in\X} \{ \phi(y-x_0) + f(y) + J^\infty(y) \}$. Then by~\eqref{eq:proof-bellman-principle-2}, there exists $y\in\X$ such that
	\begin{equation*}
		\phi(y-x_0) + f(y) + J^\infty(y) <  \phi(x_1 - x_0) + f(x_1) + \Jinf(x_1) = \Jinf(x_0)
		\;.
	\end{equation*}
	Conversely, defining $\yb\in\A_x^\infty$ such that $\yb_{\to 1} \in \mcP_{y}^\infty$, we have that by the definition of $J^\infty$ that
	\begin{align*}
		J^\infty(x) &\leq \mcR(\yb)
		\\ &= \phi(y - x) + f(y) + \mcR_{\infty}(\yb_{\to 1})
		\\ &= \phi(y - x) + f(y) + \Jinf(y)
		\;,
	\end{align*}
	which is again a contradiction, and hence $x_1 \in\argmin_{y\in\X} \{ \phi(y-x_0) + f(y) + J^\infty(y) \}$. Combining this result with~\eqref{eq:proof-bellman-principle-2}, and noting that $t$ is arbitrary we therefore conclude the proof.
\end{proof}

\subsection{Convex Analysis Results}
This section compiles some auxiliary results from convex analysis which are needed in the proofs of the main theorems below. In all proofs within this subsection, we assume without loss of generality that $\fst=0$, since we may simply consider the objective $\tilde{f}(x) = f(x) - \fst$ which satisfies this property.

\begin{lemma} \label{lem:bregman-recursion}
    Let Assumptions~\ref{ass:f-basic-assumptions},~\ref{ass:phi-assumptions-general} and~\ref{ass:f-phi-differntiable} hold and assume that $f$ is strictly convex. Assume that for any $x,y\in\X$ and $T\in\N\cup\{\infty\}$, define the dual points $q=\nabla J^T(x)$ and $p=\nabla J^T(y)$ as well as the function $\phit(x):=\phi(-x)$. We therefore have that
    \begin{equation}
        D_{(J^T)^\ast}( q , p ) = 
        D_{\phit^\ast}\left( q , p \right) +
        D_{(J^{T-1}+f)^{\ast}}( q , p )
        \;,
    \end{equation}
    as well as the recursion
    \begin{equation}
        (J^T)^\ast(q) = \phit^\ast(q) + (J^{T-1} + f)^\ast(q) \;.
    \end{equation}
\end{lemma}
\begin{proof}
	Recall that Lemma~\ref{lem:convex-J-properties-1} implies that $J^T$ is convex and differentiable for all $T\in\N$ under the assumptions of the theorem.
    Recalling the recursive properties of $\nabla J^T (x)$ and $J^T(x)$ (Lemma~\ref{lem:gradJ-recursion}), for $x,y\in\X$ and $\xb\in\mcP_x^T$ and $\yb\in\mcP_y^T$, we compute
    \begin{align}
        D_{J^T}(x,y)
        &\eqhint[eq0]{=} J^T(x_0) - J^T(y_0) - \langle \nabla J^T(y) , x_0-y_0 \rangle \notag
        \\ &\eqhint[eq1]{=} \notag
        \{ J^{T-1}(x_1) + f(x_1) - J^{T-1}(y_1) - f(y_1) - \langle \nabla J^{T-1}(y_1) + \nabla f(y_1) , x_1 - y_1 \rangle \} 
        \\& \hphantom{1em}+ \{ \phi(\Delta x_0) - \phi(\Delta y_0) - \langle \nabla J^T(y) , (x_1-x_0) - (y_1 -y_0) \rangle  \}  \notag
        \\ &\eqhint[eq2]{=} \notag
        D_{J^{T-1} + f}(x_1,y_1) + \{ \phi(\Delta x_0) - \phi(\Delta y_0) + \langle \nabla J^T(y_0) , (x_1-x_0) - (y_1 -y_0) \rangle  \}
        \\ &\eqhint[eq3]{=} \notag
        D_{J^{T-1} + f}(x_1,y_1) + \{ \phi(\Delta x_0) - \phi(\Delta y_0) - \langle \nabla \phi(\Delta y_0) , (x_1-x_0) - (y_1 -y_0) \rangle  \}
        \\ &\eqhint[eq4]{=} \label{eq:bregman-recursion}
        D_{J^{T-1} + f}(x_1,y_1) + D_\phi( \Delta x_0 , \Delta y_0)
        \;, 
    \end{align}
    where~\hintref{eq0} follows from the definition of the Bregman Divergence, where~\hintref{eq1} follows from the recursive expansion $J^T(x_0) = f(x_1) + \phi(\Delta x_0) + J^{T-1}(x_1)$ and $\nabla J^{T}(x_0)=\nabla(J^{T-1}+f)(x_1)$, where~\hintref{eq2} follows from the definition of the Bregman divergence applied to the left-most curly braces, where~\hintref{eq3} follows from the identity $\nabla J^T(x_0)=-\nabla \phi(\Delta x_0)= \nabla \phit(- 
    \Delta x_0)$ obtained from Lemma~\ref{lem:gradJ-recursion}, and where~\hintref{eq4} again follows from the definition of the Bregman divergence. 
    
    Now, recall the property of the Bregman divergence that 
    \begin{equation} \label{eq:dual-bregman-property}
        D_{g}(x,y) = D_{g^\ast}(\nabla g (y) , \nabla g( x) )
    \end{equation}
    for convex and differentiable $g$ and $g^\ast$. Since $f$ is strictly convex and differentiable and $J^{T-1}$ is convex and differentiable, $J^{T-1} + f$ is strictly convex and differntiable. Hence by~\cite[Theorem 26.3]{rockafellar1970convex} $(J^{T-1} + f)^\ast$ is also differentiable and strictly convex, so we have we have
    \begin{align*}
        D_{J^{T-1}+f}( x_1 , y_1 ) 
        &=
        D_{(J^{T-1}+f)^{\ast}}( \nabla(J^{T-1}+f)(x_1) , \nabla(J^{T-1}+f)(y_1) ) 
        \\&=
        D_{(J^{T-1}+f)^{\ast}}( \nabla J^T(x_0) , \nabla J^T(y_0) ) 
        \;,
    \end{align*}
    where in the second line, we use the recursive property that $\nabla J^T (x_0) = \nabla J^{T-1} (x_1) + \nabla f (x_1)$ from Lemma~\ref{lem:gradJ-recursion}. Applying~\eqref{eq:dual-bregman-property} and that $\nabla J^T(x_0) = \nabla \phit(- 
    \Delta x_0)$, we get
    \begin{equation*}
        D_{\phi}(\Delta x_0 , \Delta y_0) = 
        D_{\phit^\ast}\left( \nabla J^T(y_0), \nabla J^T(x_0) \right)
        \;.
    \end{equation*}
    
    Combining these results with equations~\eqref{eq:bregman-recursion} and~\eqref{eq:dual-bregman-property}, we obtain
    \begin{equation*}
        D_{(J^T)^\ast}(\nabla J^T(y),\nabla J^T(y) ) = 
        D_{\phit^\ast}\left( \nabla J^T(y_0), \nabla J^T(x_0) \right) +
        D_{(J^{T-1}+f)^{\ast}}( \nabla J^T(x_0) , \nabla J^T(y_0) ) 
        \;,
    \end{equation*}
    and hence, letting $p=\nabla J^T(x)$ and $q=\nabla J^T(y)$, we have
    \begin{align*}
        D_{(J^T)^\ast}( q , p ) = 
        D_{\phit^\ast}\left( q , p \right) +
        D_{(J^{T-1}+f)^{\ast}}( q , p )
        \;.
    \end{align*}
    Letting $p=0=\nabla J^T(\xst)$ in the above, we obtain the second result.
\end{proof}

\begin{lemma}[Bregman Relative Duality] \label{lem:interchange-relative-smoothness}
    Consider a convex, differentiable function $g:\X\rightarrow\R$. If $\phi$ is a symmetric positive-definite quadratic function on $\X$, then
    \begin{enumerate}
        \renewcommand{\labelenumi}{\roman{enumi}.}
        \item If $g$ is $\lambda$-relatively-smooth with respect to $\phi$, then $g^\ast$ is $\frac{1}{\lambda}$-relatively-convex with respect to $\phi^\ast$.
        \item If $g$ is $\mu$-relatively-convex with respect to $\phi$, then $g^\ast$ is $\frac{1}{\mu}$-relatively-smooth with respect to $\phi^\ast$.
    \end{enumerate}
\end{lemma}
\begin{proof}
    We first note that \emph{i} and \emph{ii} are identical statements, where we obtain the other by interchanging $g$ and $g^\ast$. Hence, we only show the proof of \emph{i}.
    
    We begin by establishing a few results regarding $\phi$. First, since $\phi$ is symmetric positive-definite and quadratic, there exists a norm $\| \cdot \|_0$ on $x$ such that $\phi(x) = \frac{1}{2} \| x \|^2_0$. This fact is easy to verify by setting $\| \cdot \|_0 = \sqrt{2 \phi(x)}$, and verifying that this satisfies the necessary conditions of a norm.
    Next, it is also easy to verify (either by simple computation or see~\cite[Example 3.27]{boyd2004convex}) that $\phi^\ast(p) = \frac{1}{2}\| p \|_{0,\ast}^2$, where $\| \cdot \|_{0,\ast}$ represents the dual norm of $\| \cdot \|_0$. 

    Since $\phi$ is quadratic, we also have that $D_\phi(x,y)=\frac{1}{2} \| x -y \|^2_0$, and therefore the definition of relative smoothness and strong smoothness with respect to $\|\cdot\|_0$ of~\cite[Definition 5]{kakade2009applications} co-incide. Hence, $g$ is $\lambda$-strongly-smooth with respect to the norm $\| \cdot \|_{0}$. Applying~\cite[Theorem 6]{kakade2009applications}, we obtain that $g^\ast$ must be $\frac{1}{\lambda}$-strongly-convex with respect to the norm $\|\cdot\|_{0,\ast}$, and hence $g^\ast$ is $\frac{1}{\lambda}$-relatively-convex with respect to $\phi^\ast$, yielding the desired result.
\end{proof}

\subsection{ Descent Lemmas }

\begin{theorem}[Primal Descent Lemma]
	\label{lemma:primal-descent-lemma}
	Let Assumptions~\ref{ass:f-basic-assumptions},~\ref{ass:phi-assumptions-general} and~\ref{ass:f-phi-differntiable} hold, and let $f$ be strictly convex. For $T\in\N\cup\{\infty\}$, let $\xb \in \mcP_{x}^T$ and assume that $\phi$ is a symmetric positive-definite quadratic function on $\X$, and that $J^T$ is 1-relatively-smooth with respect to $\phi$. For $0 \leq t \leq T-1$, we have
	\begin{enumerate}
		\item For all $y \in \D$, $J^{T}(x_{t+1}) - J^{T}(y) \leq 
		- D_{J^{T}}( y \mathrel{,} x_t) +
		D_{\phi}( y \mathrel{,} x_t)  - D_{\phi}( y  \mathrel{,} x_{t+1} )$,
	\end{enumerate}
	from which it follows that
	\begin{enumerate}
		\setcounter{enumi}{1}
	 	\item $J^{T}(x_{t+1}) - J^{T}(x_t) \leq  
	 	- D_{\phi}( x_t  \mathrel{,} x_{t+1} )$, and
	 	\item $D_{J^T}( x_{t+1} \mathrel{,} \xst ) \leq 
		- D_{J^{T}}( \xst \mathrel{,} x_t) +
		D_{\phi}( \xst \mathrel{,} x_t)  - D_{\phi}( \xst  \mathrel{,} x_{t+1} )$.
	\end{enumerate} 
\end{theorem}
\begin{proof}
	We first note that by Lemma~\ref{lem:convex-J-properties-2}, the assumptions of the theorem guarantee that $J^T$ is Legendre convex.
    We first note that since $\phi$ is quadratic, we have the properties that $\phi(x) = \phi(-x)$.
	Due to the assumed linearity of $\nabla \phi$, we find that the update rule for $\xb\in\mcP_{t,x}^T$ can be expressed as
	\begin{align*}
		\nabla \phi \left( x_{t+1} - x_t \right) 
		&= \nabla \phi \left( x_{t+1} \right) - \nabla \phi \left( x_{t+1} \right) 
		\\ &= - \nabla J^{T}(x_t)
		\;,
	\end{align*}
	where the second equality holds due to the linearity of $\nabla \phi$. It is then easy to verify that this update rule can also be expressed as the proximal update rule
	\begin{equation}
		x_{t+1} = \arg \min_{z} \left\{
		\langle \nabla J^{T}(x_t) \mathrel{,} z - x_t \rangle
		+ D_{\phi}\left( z \mathrel{,} x_t \right)
		\right\}
	\end{equation}
	which for any $z\in\D$ satisfies the proximal inequality
	\begin{equation}
		\langle \nabla J^{T}(x_t) \mathrel{,} z - x_t \rangle +
		D_{\phi}\left( z \mathrel{,} x_t \right)
		\geq
		\langle \nabla J^{T}(x_t) \mathrel{,} x_{t+1} - x_t \rangle
		+ D_{\phi}\left( x_{t+1} \mathrel{,} x_t \right)
		+ D_{\phi}\left( z \mathrel{,} x_{t+1} \right)
		\;.
	\end{equation}
	Hence, we have that
	\begin{align*}
		J^{T}(x_{t+1}) 
		&\eqhint[leq1]{\leq} 
		J^{T}(x_{t}) + \langle \nabla J^{T}(x_t) \mathrel{,} x_{t+1} - x_t  \rangle +
		D_{\phi}( x_{t+1} \mathrel{,} x_t )
		\\ &\eqhint[leq2]{\leq}
		J^{T}(x_{t}) + 
		\langle \nabla J^{T}(x_t) \mathrel{,} x - x_t  \rangle + 
		D_{\phi}(x \mathrel{,} x_t)  - D_{\phi}(x  \mathrel{,} x_{t+1} )
		\\ &\eqhint[leq3]{\leq}
		J^{T}(x) - D_{J^{T}}(x \mathrel{,} x_t) +
		D_{\phi}(x \mathrel{,} x_t)  - D_{\phi}(x  \mathrel{,} x_{t+1} )
		\;,
	\end{align*}
	where~\hintref{leq1} follows from~\ref{lem:convex-J-properties-1}-\emph{i},~\hintref{leq2} follows from the proximal inequality and~\hintref{leq3} follows from simple algebra, yielding the first result in the statement of the lemma.\reseteqhint The second and third follow by applying the special cases $y=x_t$ and $y=\xst$ to the first.
\end{proof}

\begin{theorem}[Dual Descent Lemma] \label{lemma:dual-descent-Lemma}
	Let Assumptions~\ref{ass:f-basic-assumptions},~\ref{ass:phi-assumptions-general} and~\ref{ass:f-phi-differntiable} hold, and let $f$ be strictly convex. For $T\in\N\cup\{ \infty \}$, let $\xb\in\mcP^T$ and assume that $(J^T)^\ast$ is $\lambda$-relatively-convex with respect to $\phi$. 
	\begin{enumerate}
		\item For all $y\in\D$ and $0\leq t \leq T-1$, and letting $\phit(x) := \phi(-x)$, we have
	\begin{align*} 
	\label{eq:dual-descent-lemma-0}
		\phit^{\ast}\left( 
			\nabla J^{T}( x_{t+1} )
		\right)
		-
		\phit^{\ast}\left( 
			\nabla J^{T}( y ) 
		\right)
		\leq
		- & D_{\phit^{\ast}}\left( \nabla J^T(y) \mathrel{,} \nabla J^T(x_t) \right)
		\\&\hspace{1.5em}+ 
		D_{J^T}\left( x_{t}\mathrel{,} y \right) - D_{J^T}\left( x_{t+1}\mathrel{,} y \right) 
		\;.
	\end{align*}
	\end{enumerate}
	Moreover, the above inequality implies that for all $0\leq t \leq T-1$,
	\begin{enumerate}
		\setcounter{enumi}{1}
		\item $\phit^{\ast}\left( 
			\nabla J^{T}( x_{t+1} )
		\right)
		+ D_{J^T}\left( x_{t+1} \mathrel{,} x_t \right)
		\leq
		\phit^{\ast}\left( 
			\nabla J^{T}( x_t ) 
		\right)
		$, and
		\item For any $x^{\star} \in \arg\min_{y\in\D} J^T(y)$,
		\begin{align*}
		D_{J^T}\left( x_{t+1}\mathrel{,}\xst \right) 
		+
		D_{\phi^{\ast}}\left( 
		\nabla J^{T}(x_{t+1}) 
		\mathrel{,}
		\nabla J^{T}(\xst ) \right) &
		\\+
		D_{\phi^{\ast}}\left( \nabla J^T( \xst ) \mathrel{,} \nabla J^T(x_t) \right)
		&\leq
		D_{J^T}\left( x_{t}\mathrel{,} \xst \right)
		\end{align*}
	\end{enumerate}
\end{theorem}
\begin{proof}
    We first note that by Lemma~\ref{lem:convex-J-properties-2}, the assumptions of the theorem guarantee that $J^T$ is Legendre convex. Now let $k=\phit^\ast$, $f=J^T$ and $L^\ast=1$ in the statement of~\cite[Lemma 4.6]{maddison2019dual}. Noting that a single iterate of $\xb\in\mcP_x^T$ coincides exactly with the iterates of \cite[Algorithm 1.1]{maddison2019dual} with $L^\ast = 1$, we apply~\cite[Lemma 4.6]{maddison2019dual} to obtain the desired result. 
\end{proof}

\subsection{ Descent on Convex Objectives }

\begin{lemma} \label{lem:lagrangian-decreasing}
    Let $f$ be convex, differentiable and $\phi$ be a symmetric positive-definite quadratic function, and consider $\xb\in\mcP^\infty$. Let us define the Lagrangian
    \begin{equation*}
        \L_t = f(x_{t+1}) + \phi(\Delta x_t)
        \;.
    \end{equation*}
    Then $\L_t$ is a non-increasing sequence.
\end{lemma}
\begin{proof}
	We first note that by Lemma~\ref{lem:convex-J-properties-2}, the assumptions of the theorem guarantee that $\Jinf$ is Legendre convex.
    We first note that since $\phi$ is quadratic, we have the properties that $\phi(x) = \phi(-x) = \frac{1}{2}\langle x , \nabla \phi(x) \rangle$ for all $x\in\X$, and that $\phi^\ast$ is quadratic as well.
    We begin by separating the expression for $\L_t$ into two parts and showing that each is monotone. Using the short-hand notation $\nabla \Jinf(x_t) = \nabla \Jinf_t$, we separate the expression as
    \begin{align}
        \L_t &= 
        \{ f(x_{t+1}) - \phi^\ast( \nabla \Jinf_{t+1} )  \} +  \phi^\ast( \nabla \Jinf_{t+1} ) + \phi(\Delta x_t) \notag
        \\ &=
        \mcH_{t+1} + \phi^\ast( \nabla \Jinf_{t+1} ) + \phi(\Delta x_t)
        \;.  \label{eq:prf-lagrangian-decreasing-1}
    \end{align}
   
    Note that since $-\nabla \phi(\Delta x_t) = \nabla \Jinf_t$ and that  and hence that $\Delta x_t = \nabla \phi^\ast( -\nabla \Jinf_t )$, we have 
    \begin{align*}
        \phi(\Delta x_t) &= 
        \langle \nabla \phi(\Delta x_t) , \Delta x_t \rangle - \phi^\ast( \nabla \phi(\Delta x_t) )
        \\ &= 
        \langle \nabla \phi(\nabla \phi^\ast( -\nabla \Jinf_t )) , \nabla \phi^\ast( -\nabla \Jinf_t ) \rangle - \phi^\ast( \nabla \phi(\Delta x_t) )
        \\ &=
        \langle \nabla \Jinf_t , \nabla \phi^\ast(\nabla \Jinf_t) \rangle - \phi^\ast( \nabla \Jinf_t )
        \;.
    \end{align*}
    Since $\phi^\ast$ is quadratic, $\phi^\ast(x) = \frac{1}{2}  \langle \nabla \phi^\ast(x) , x \rangle $, and hence we conclude that
    \begin{equation*}
        \phi(\Delta x_t) = \phi^\ast(\nabla J_t)
        \;.
    \end{equation*}
    By Lemma~\ref{lemma:dual-descent-Lemma}-\emph{ii}, we have that $\phi^\ast(\nabla J_{t+1}) + \phi(\Delta x_t) = \phi^\ast( \nabla \Jinf_{t+1} ) + \phi^\ast( \nabla \Jinf_{t} )$ is a monotone non-increasing sequence, and hence the latter two terms of equation~\eqref{eq:prf-lagrangian-decreasing-1} are monotone non-increasing.
    
    To show that $\mcH_t$ is non-increasing, we use that $\Delta x_t = \nabla \phi^\ast( -\nabla \Jinf_t )$ and that $\nabla \Jinf_{t} = \nabla f(x_{t+1}) + \nabla \Jinf_{t+1}$ to compute
    \begin{align*}
        \mcH_{t+1} - \mcH_t &=
        \{ f(x_{t+1} - f(x_t) \} - \{ \phi^\ast(\nabla \Jinf_{t+1}) - \phi^\ast(\nabla \Jinf_{t}) \}
        \\ &= - \left\{ D_f(x_{t},x_{t+1}) + D_{\phi^\ast}(\nabla \Jinf_{t+1} , \nabla \Jinf_{t} )  \right\} - \langle \Delta x_t, \nabla f(x_{t+1}) \rangle + \langle \nabla \Jinf_{t+1} - \nabla \Jinf_{t} , \phi^\ast( \nabla \Jinf_{t} ) \rangle
        \\ &= - \left\{ D_f(x_{t},x_{t+1}) + D_{\phi^\ast}(\nabla \Jinf_{t+1} , \nabla \Jinf_{t} )  \right\} - \langle \Delta x_t, \nabla f(x_{t+1}) \rangle - \langle \nabla \Jinf_{t+1} - \nabla \Jinf_{t} , \Delta x_t \rangle
        \\ &= - \left\{ D_f(x_{t},x_{t+1}) + D_{\phi^\ast}(\nabla \Jinf_{t+1} , \nabla \Jinf_{t} )  \right\} 
        \\&\leq 0
        \;.
    \end{align*}
    Hence $\mcH_t$ is non-increasing and we have the desired result.
\end{proof}

\newpage

\section{
{Proofs for Section~\ref{sec:existence-and-consistency}}
}
Most of the results in this section, rely on establishing the interchangeability of the $\lim$ and $\argmin$.  In general, these operations do not commute.  The theory of \textit{$\Gamma$-convergence}, introduced in \cite{DeGiorgi}, is designed specifically to address these types of problems.  We briefly overview this theory before applying its to derive proofs of the results in Section~\ref{sec:existence-and-consistency}.  
\subsection{Overview of  {\texorpdfstring{$\Gamma$}{Γ}-Convergence Theory}}\label{ss_appendix_gamma_convergence_overview}

Fix a metric space $(X,d)$ and consider a functional $F:X\rightarrow (-\infty,\infty] \cup \{\infty\}$ which we would like to minimize on $X$.   When it exists, a minimum of $F$ describes an $x \in X$ achieving the lowest value of $F$, or equivalently, it represents the lowest point on F's epigraph, as defined by 
$\operatorname{epi}(F):= \left\{(x,r)\in X\times (-\infty,\infty]:\, r\geq  F(x)\right\}$; this set describes all points on or above F's graph.  

We can expect F to admit a minimum precisely if the sub-level set of $\operatorname{epi}(F)$ are not too large, if limits of points in $\operatorname{epi}(F)$ stay within $\operatorname{epi}(F)$, and if $\operatorname{epi}(F)$ does not trail off to $-\infty$ at some point.  These three conditions respectively describe the intuition behind the need for F's coercivity, lower semi-continuity, and boundedness from below.  Formally, a function $F$ is said to be \textit{coercive} if its sub-level sets $\operatorname{Lev}_s(F):= \left\{x\in X:\,F(x)\leq s \right\}$, for every $s\geq \R$, are compact in $X$, $F$ is said to be \textit{lower semi-continuous} if $\operatorname{epi}(F)$ is a closed subset of $X\times (-\infty,\infty]$, and it is bounded from below if $F(x)\geq M$ for some $M\in \R$.  The \textit{Direct Method} of \cite{Tonelli}, guarantees that together these three conditions are sufficient for $F$ to be minimized over $X$.  

Suppose now that $F$ can be described as the point-wise limit of a sequence of functionals $(F_n)_{n \in \N}$ on $X$.  It is tempting, to solve $\operatorname{argmin}_{x \in X} F(x)$ by interchanging the limit and $\operatorname{argmin}$ operations via
\begin{equation}
    \lim\limits_{n\uparrow \infty} 
    \operatorname{argmin}_{x \in X} F_n(x)
        =
    \operatorname{argmin}_{x \in X} 
    \lim\limits_{n\uparrow \infty} F_n(x)
    ;
    \label{eq_gamma_convergence_motivation}
\end{equation}
however,~\eqref{eq_gamma_convergence_motivation} is generally false even when the convergence of $F_n$ to $F$ is uniform (see \cite{DalMasoGamma1993}).  This is because any of the three aforementioned properties can fail for the limiting functional $F$ even if they hold for each of the $F_n$.  

Any inconsistency with the lower semi-continuity of the functionals $(F_n)_{n \in \N}$ and $F$ is avoided when the epigraphs $\operatorname{epi}(F_n)$ converge to F's epigraph.  There are various modes of convergence of sets, see \cite{AubinSetValued}; however, the correct notion of set-convergence here is \textit{Kuratowski convergence}, introduced in \cite{kuratowski1966topology}, which intuitively describes the set of accumulation points of $\operatorname{epi}(F_n)$ and is defined to be
$
\{
(x,t)\in X\times (-\infty,\infty]:\, \underset{n\uparrow \infty}{\limsup}\, d((x,t),\operatorname{epi}(F_n))=0
\},
$
whenever that set coincides with $
\{
(x,t)\in X\times (-\infty,\infty]:\, \underset{n\uparrow \infty}{\liminf}\, d((x,t),\operatorname{epi}(F_n))=0
\}
.
$  If this happens, we say that the sequence of functionals $(F_n)_{n \in \N}$ $\Gamma$-converges to $F$ and we write $\Gamma-\lim_{n\uparrow \infty} F_n = F$.   

\begin{remark}\label{rem_general_gamma}
$\Gamma$-convergence can still be formulated when X is not a metric space, for example, this is the case when $X$ is an infinite-dimensional Banach space such as $\A_x^{\infty:0}$ equipped with its weak topology.  For details on this general case, we point the reader to \citep[Chapter 4]{DalMasoGamma1993}.  
\end{remark}

Analogously, any inconsistency with the coercivity of the functionals $(F_n)_{n\in \N}$ is avoided when the sequence is \textit{equi-coercive}.  This means that the sub-level sets are uniformly small in the same places; mathematically, this means that for every $s \in \R$ there exists a compact subset $K$ of $X$ containing each $\operatorname{Lev}_s(F_n)$, for $n \in \N$.  

When the sequence $(F_n)_{n \in \N}$ is both equi-coercive and its $\Gamma$-limit is $F$, then the \textit{Fundamental Theorem of $\Gamma$-Convergence} \citep[Theorem 2.1]{BraidesLocalMinimization} is a sequential extension of Tonelli's Direct method.  Accordingly, the result guarantees that $F$ and each $F_n$ is minimized over $X$ and that any accumulation point of these minima are a minimizer of $F$.   

\subsection{Proof of Theorem~\ref{thm:finite-existence}}\label{s_proof_Theorem_finite_Existence}
We establish Theorem~\ref{thm:finite-existence} using Tonelli's Direct method, described in the previous section.  This amounts to showing that the functional $\mcR_T$, for each $T\in \N$, is lower semi-continuous, coercive, and bounded-below.  

We simplify our task by breaking up the regret functional into the sum $\mcR_T = F_T + \Phi_T$, where the functionals $\Phi_T$ and $F_T$ on $\A_x^T$ are defined by
$$
\begin{aligned}
\Phi_T: \xb \mapsto \sum_{t=1}^{T} \phi(\Delta x_t) \mbox{ and } &
F_T : \xb \mapsto \sum_{t=1}^T f(x_t)-f^{\star}.
\end{aligned}
$$
The convenience arises from the fact that we can establish each of these two functionals' individual properties, since they rely on different assumptions, before combining them back together to infer the relevant properties of $\mcR_T$.  Accordingly, our task is divided into establishing a sequence of lemmas, each dedicated to showing a different property of $F_T$ or $\Phi_T$, at which point theorems' proof reduces Tonelli's method.  

\subsubsection{Auxiliary Lemmas}
We now note that $\A_x^{\infty,0}$ will be equipped with the topology generated by the norm $\xb \mapsto ( \sum_{t=1}^{\infty}\|\Delta x_t\|^p )^{\frac{1}{p}}$ and we note that $\A_x^{\infty:0}$ is a Banach space which is isometrically isomorphic to $\ell^p(X)$ via the map $\xb \mapsto \Delta \xb$.  We also observe that $\A_x^{\infty:\alpha} \subset \A_x^{\infty:0} \subset \A_x^{\infty}$ for every $\alpha>0$.
\begin{lemma}[Regularity of Algorithmic Penalty Function]\label{lem_assumption_translation}
    Under Assumption~\ref{ass:phi-assumptions-general}, for each $T\in \N\cup\{\infty\}$, $\Phi_T$ is coercive on each $\A_x^T$ (resp. $\A_x^{\infty:0}$ when $T=\infty$).  Moreover, it is weakly lower semi-continuous, not identically $\infty$, and bounded-below by $0$ on each $\A_x^T$ with $T<\infty$ (resp. on $\A_x^{\infty:\alpha}$ for $\alpha\geq 0$ when $T=\infty$ ).  In particular, it is lower semi-continuous on $\A_x^{\infty:\alpha}$ for $\alpha\geq 0$.  
    Furthermore, the following hold:
    \begin{enumerate}
        \item $\Phi_{\infty}$ is coercive on $\A_x^{\infty:\alpha}$, for every $\alpha>0$,
        \item $\Phi_{\infty}$ is weakly lower semi-continuous, and weakly coercive on $\A_x^{\infty:0}$.  
    \end{enumerate}
\end{lemma}
\begin{proof}

Since the weak lower semi-continuity of any function from $\A_x^{\infty:0}$ to $(-\infty,\infty]$ implies its lower semi-continuity (l.s.c.), then it is enough for us to establish the former on $\A_x^{\infty:0}$.  This is because, the restriction of weakly lower semi-continuous (weakly l.s.c.) functions to closed sets of a topological space preserves weak lower semi-continuity; thus, it is enough to show that $\Phi_{\infty}$ is l.s.c. on $\A_x^{\infty:0}$ to conclude that it must also be l.s.c. on each $\A_x^{\infty:\alpha}$ for $\alpha\geq 0$.

We begin with the following remark.  Since $\X$ is a finite-dimensional normed space, it must admit a Hamel basis $\{e_n\}_{n=1}^{\dim(\X)}$ such that every $x \in \X$ is uniquely expressed as $x = \sum_{n=1}^{\dim(\X)} \beta_n e_n$.  In particular, for $1\leq n\leq \dim(\X)$, the map $x \mapsto \pi_n(x)=\beta_n$ mapping $x$ to the coefficient $\beta_n$ in its Hamel basis expansion is a bounded linear map and it is therefore, continuous.
Next, for every $T \in \N$, define the map $p_T:\A_x^{\infty:0}\rightarrow \X$,
taking $\xb$ to its value at the $T$-th component, $x_T$. By definition this is a bounded linear map.
Hence, for every $T \in \N$ and every $1\leq n\leq \dim(\X)$, 
the composition $p_T\circ \pi_n:\A_x^{\infty:0}\rightarrow \R$ is a bounded linear functional.  
Moreover, for every $T\in \N$, we have the representation $p_T = \sum_{n=1}^{\dim(\X)} p_T \circ \pi_n e_n$.

Since $\A_x^{\infty:0}$ is a Banach space then, by definition, in the weak topology on $\A_x^{\infty:0}$ all bounded linear functionals are continuous.  Thus, for $T \in \N$ and every $1\leq n\leq \dim(\X)$, the map sending $p_T\circ \pi_n$ is bounded and linear then it is weakly continuous on $\A_x^{\infty:0}$.  Now, since the sum of weakly lower continuous functions is again weakly lower continuous and since $p_T = \sum_{n=1}^{\dim(X)} p_T e_n$, then $p_T$ is weakly-to-strong continuous\footnote{Let $X$ be a topological space endowed with the weak topology and $Y$ be a normed space. We say that a map $f:X \rightarrow Y$ is \emph{weak-to-strong} continuous if it satisfies the usual definition of continuity; that is, for any point $x\in X$ and neighborhood $\epsilon_y$ of $f(x) = y \in Y$, there exists a neighborhood $\delta_x$ of $x$ such that $f(\delta_x)\subseteq \epsilon_y$. } from $\A_x^{\infty:0}$ to $\X$.

Next, since the pre-composition of a lower semi-continuous function by a weak-to-strong continuous map is weakly lower semi-continuous then, for every $T \in \N$, the map $\phi\circ p_T$ is weakly lower semi-continuous.  
By \citep[Proposition 1.9]{DalMasoGamma1993} the sum of weakly lower semi-continuous functions is again weakly lower semi-continuous and therefore, for every $T \in \N$, $\Phi_T:\A_x^{\infty:0}\rightarrow (-\infty,\infty]$ is weakly lower semi-continuous.  Since the point-wise supremum of a family of weakly lower semi-continuous functions is weakly lower semi-continuous, see \citep[Proposition 1.8]{DalMasoGamma1993}, then $\Phi_{\infty}=\sup_{T\in \N} \Phi_T$ is weakly lower semi-continuous on $\A_x^{\infty:0}$.


Now, consider the sequence $\xb^{\star}$ defined by $\xb^{\star}_t = x$.  Since $\phi(0)=0$ and $\Delta \xb^{\star}_t=0$ for every $t\in \N$ then $\Phi$ is not identically $\infty$ on $\A_x^{\infty:0}$.  Moreover, by construction $\Phi\geq 0$.

By \citep[Definition 1.12]{DalMasoGamma1993}, it $\Phi$ is coercive on $\A_x^{\infty:\alpha}$ if its sub-level sets are compact; i.e.\: for every $s\geq 0$ the sub-level set $\operatorname{Lev}_s(\Phi):= \left\{
\xb \in \A_x^{\infty:\alpha}:\, \Phi(\xb)\leq s
\right\}$.  
Assumption~\ref{ass:phi-assumptions-general} implies that for every $s\geq 0$, we have the inclusion
\begin{equation}
    \operatorname{Lev}_s(\Phi)
\subseteq 
\left\{
\xb \in \A_x^{\infty:0}:\, c\sum_{t=1}^{\infty}n^{\alpha}\|\Delta x_t\|^p \leq s
\right\}
\label{compactness_set}
.
\end{equation}
Since $\A_x^{\infty:0}$ is isometrically isomorphic to $\ell^p(X)$ via the map $\xb \mapsto \Delta \xb$ then Grothendieck's compactness principle (here we use the formulation of \citep[Exercises 1.6]{SequencesSeriesBanach}) implies that the right-hand side of~\eqref{compactness_set} is a compact subset of the Banach space $\A_x^{\infty:0}$; which, by construction, is a subset of $\A_x^{\infty:\alpha}$.  Since we assume that $\Phi$ is lower semi-continuous, then the sub-level set $\operatorname{Lev}_s(\Phi)$ is closed and in particular it is a closed subset of the compact set $\left\{
\xb \in \A_x^{\infty:0}:\, c\sum_{t=1}^{\infty}n^{\alpha}\|\Delta x_t\|^p \leq s
\right\}$; thus $\operatorname{Lev}_s(\Phi)$ is compact by \citep[Theorem 26.2]{munkres2000topology}.

Consider the case where $\phi$ is convex.Let $\xb,\yb \in \A_x^{\infty:0}$ and $\rho \in [0,1]$ then the convexity of $\phi$ and the linearity of $\Delta$ imply that
$$
\begin{aligned}
\sum_{t=1}^{\infty} \phi\left(
\Delta \left(
\rho x_t + (1-\rho) y_t
\right)
\right)
= &
\sum_{t=1}^{\infty} \phi\left(
\rho \Delta x_t + (1-\rho) \Delta y_t
\right)
\\
\leq &
\sum_{t=1}^{\infty} 
\rho \phi(\Delta x_t) + (1-\rho) \phi(\Delta y_t);
\end{aligned}
$$
hence $\Phi_{\infty}:\xb \mapsto \sum_{t=1}^{\infty} \phi(\Delta x_t)$ is convex on $\A_x^{\infty:0}$.  

Next, since there exists a constant $c>0$ satisfying $c\|x\|^p \leq \phi(x)$ for every $x \in \X$ and since $\X$ is a linear space then $\Delta x_t\in X$ for every $\xb \in \A_x^{\infty:0}$.  Thus, $c\|\Delta x_t\|^p \leq \phi(\Delta x_t)$ for every $\xb \in \A_x^{\infty:0}$ and therefore
$$
c \|\xb\|_{\A_x^{\infty:0}}^p = c \sum_{t=1}^{\infty}\|\Delta x_t\|^p\leq \sum_{t=1}^{\infty} \phi(\Delta x_t) = \Phi(\xb);
$$
whence $\Phi_{\infty}$ is weakly coercive, since if $\|\xb\|_{\A_x^{\infty:0}} \to \infty$ implies that $\Phi_{\infty}(\xb)\to \infty$.  
\end{proof}
\begin{lemma}[Regularity of Unregularized Regret]\label{lem_lscty_f}
Under Assumption~\ref{ass:f-basic-assumptions}, for every $T \in \N \cup\{\infty\}$, the functions $F_T(\xb):= \sum_{t=1}^{T} f(\xb)$ are weakly lower semi-continuous on $\A_x^{\infty:\alpha}$, for every $\alpha\geq 0$.  In particular, they are lower semi-continuous on $\A_x^{\infty:\alpha}$, for every $\alpha\geq 0$.
Moreover, each $F_T$ is weakly lower semi-continuous on $\A_x^{\infty:0}$.  
\end{lemma}
\begin{proof}
Analogously to the proof of Lemma~\ref{lem_assumption_translation}, it is enough for us to show that, for each $T \in \N$, the function $F_T$ is weakly lower semi-continuous on $\A_x^{\infty:0}$ to conclude that it is both lower semi-continuous and weakly lower semi-continuous on each $\A_x^{\infty:\alpha}$ for every $\alpha \geq 0$.  

As in the proof of Lemma~\ref{lem_assumption_translation}, we know that each map $p_n:\A_x^{\infty:0}\rightarrow \X$ weakly continuous.  Since $f$ is lower semi-continuous then, for each $T \in \N$, the composition $f\circ p_T:\A_x^{\infty:0}\rightarrow (-\infty,\infty]$ is weakly lower semi-continuous.  Applying \citep[Proposition 1.9]{DalMasoGamma1993} we conclude that, for every $T\in \N$, $F_T$ is also weakly lower semi-continuous.  By \citep[Proposition 1.8]{DalMasoGamma1993}, since the supremum of any family of lower semi-continuous functions is itself lower semi-continuous; thus,
    $$
    F_{\infty}(\xb)=\sum_{t=1}^{\infty} f(x_t)-f(x^{\star}) 
    = \sup_{T \in \N}
    \sum_{t=1}^{T} f(x_t)-f(x^{\star}) 
    = \sup_{T \in \N} F_T(\xb)
    ,
    $$
    is weakly lower semi-continuous.
    
    If in addition, if $f$ is convex then arguing similarly to the proof of Lemma~\ref{lem_assumption_translation} we find that each $F_T$, for $T \in \N \cup\{\infty\}$, is convex.  Since the point-wise supremum of any family of convex functions is itself convex, then $F_{\infty}$ is also convex if $f$ is convex.
\end{proof}
\begin{lemma}[Regularity of Regret Functionals]\label{lem_bound_coercive}
Under Assumptions~\ref{ass:f-basic-assumptions} and~\ref{ass:phi-assumptions-general}, for every $T\in \N$, the regret functionals $\mcR_T=F_T+\Phi_T$ are not identically $\infty$, bounded-below by $0$, and weakly coercive on $\A_x^{\infty:0}$ and weakly lower semi-continuous on $\A_x^{\infty:0}$.  
\end{lemma}
\begin{proof}
        Since $f(x)\geq f^{\star}$ for every $x \in X$ and since $\Delta x_t \in X$ for any $\xb \in \A_x^{\infty}$ then $f(x_t)-f^{\star}\geq 0$ for every $\xb \in \A_x^{\infty}$.  Thus, each $F_T$ takes values in $[0,\infty]$.  Moreover, since $\phi(x^{\star}-x)<\infty$ then the sequence $\zb$ given by $\zb_0=x$, $\zb_t=x^{\star}$ for $t>0$ satisfies $(F_T+\Phi_T)(\zb) = f(x)+\phi(x^{\star}-x)<\infty$.  Hence, each $F_T+\Phi_T$ is not identically $\infty$ on $\A_x^{\infty:0}$.  
    Moreover, since, for every $T\in \N$, $F_T$ and $\Phi_T$ are weakly lower semi-continuous on $\A_x^{\infty:0}$ then \citep[Proposition 1.9]{DalMasoGamma1993} implies that $\mcR_T$ is weakly lower semi-continuous.  
    
    It remains only to demonstrate the weak coercivity of the regret functionals.  
    We denote the sub-level set at $s \in \R$ of any $G:\A_x^{\infty:0}\rightarrow \R \cup \{\infty\}$, by $\operatorname{Lev}_s(G):= 
    \left\{
    \yb \in \A_x^{\infty:0}:\, G(\yb)\leq s
    \right\}.
    $  Since each $F_{T}$ and $\Phi_{T}$ are bounded-below by $0$, then $(F_T+\Phi_T)(\yb)\geq \Phi_{T}(\yb)$ for every $\yb \in \A_x^{\infty:0}$.  Hence, $
    \operatorname{Lev}_s(F_T+\Phi_T)
    \subseteq 
    \operatorname{Lev}_s(\Phi_T)$ for every $s>0$ and every $T\in \N$.  
    Moreover, since $F_T$ and $\Phi_T$ are lower semi-continuous then, for every $s>0$ and every $T\in \N$, each $\operatorname{Lev}_s(F+\Phi)$ (resp. $\operatorname{Lev}_s(F_T+\Phi)$) is a weakly closed subset of $\operatorname{Lev}_s(\Phi_T)$.  
    Since $\Phi_T$ is weakly coercive then, by definition, each $\operatorname{Lev}_s(\Phi)$ is weakly compact.  Since the weak-topology is metrizable on every compact subset thereof, and since every closed subset of a compact set is itself compact in a metric space then $\operatorname{Lev}_s(F_T+\Phi_T)$ is weakly-compact; therefore, $\mcR_T=F_T+\Phi_T$ is weakly coercive on $\A_x^{\infty:0}$.
    %
\end{proof}
\subsubsection{Proof of Theorem~\ref{thm:finite-existence}}\label{sss_appendix_proof_theorem_finite_Existence}
\begin{proof}
    By Lemmas~\ref{lem_lscty_f} and~\ref{lem_assumption_translation} each $F_T$ is lower semi-continuous on $\A_x^{\infty:\alpha}$ (resp. weakly lower semi-continuous on $\A_x^{\infty:0}$).  By construction, $F_T\geq 0$ and by Lemma~\ref{lem_assumption_translation}, $\Phi_T$ is bounded-below and not identically $\infty$.  By Assumption~\ref{ass:f-basic-assumptions}, for $T\in \N$, $F_T$ is not identically $\infty$ either since the sequence $\xb^T$ defined by $\xb^T_0:= x$, $\xb^T_t:= x^{\star}$ for $0<t\leq T$, satisfies
    \begin{equation}
        (F_T + \Phi_T)(\xb^T) = f(x) + \phi(x^{\star}-x)<\infty
        \label{thm:finite-existence_proof_nice}
        .
    \end{equation}
    By Lemma~\ref{lem_bound_coercive}, for every $T \in \N$, $F_T+\Phi_T$.  
    Therefore, for every $T \in \N$, $F_T+\Phi_T$ is bounded-below by $0$, lower semi-continuous, and coercive on $\A_x^{\infty:\alpha}$;
    hence, by \citep[Theorem 1.15]{DalMasoGamma1993} each $\mcP_x^{\infty:0}$ is non-empty.  Now since $\mcR_T$ is point-wise monotonically increasing in $T$, then $\xb=(x_t)_{t \in \N} \in \mcP^{\infty:0}$ 
    only if the halted sequence 
    $$
    \tilde{\xb}\triangleq 
    \begin{cases}
    x_t & :\, t\leq T\\
    0 & :\, \mbox{else}
    \end{cases},
    $$
    belongs to $\mcP_x^T$.  Hence, $\mcP_x^T$ is non-empty.  
    Moreover, for each $T \in \N$, since $F_T+\Phi_T$ is not identically $\infty$ then by definition of $\mcP_x^{\infty:0}$ and~\eqref{thm:finite-existence_proof_nice} we have any $\xb^{\star:T}\in \mcP_x^T$ satisfies
    $$
    (F_T+\Phi_T)(\xb^{\star:T})
    \leq 
    (F_T+\Phi_T)(\xb^{T}) <\infty
    .
    $$
    Hence, $\mcP_x^T$ is non-empty and every element therein has finite regret.  
\end{proof}

\subsection{Proof of Theorem~\ref{thm:asym-existence}}
The results of Theorems~\ref{thm:asym-existence} and~\ref{thm:asym-consistency} both follow as consequences of the Fundamental Theorem of $\Gamma$-convergence, described in Section~\ref{ss_appendix_gamma_convergence_overview}.  As in the proof of Theorem~\ref{thm:finite-existence}, we begin by establishing the required properties which will allow to apply this result.  These amount to showing that the regret functional $\mcR_{\infty}$ is the $\Gamma$-limit of an equi-coercive family of functionals expressing the finite-time regret-optimal algorithm selection problem posed on each $\A_x^T$ embedded within the larger space $\A_x^{\infty:0}$.  

We now precisely define the process of embedding the finite-time-horizon control problems as control problems on $\A_x^{\infty:0}$. For any $T\in \N\cup\{\infty\}$, we introduce the \textit{indicator functions}\footnote{We note here that we use the \emph{indicator function} terminology belonging to convex-analysis and not that of probability theory.} $\chi_{\A_x^T}:\A_x^{\infty:0}\rightarrow \R \cup \{\infty\}$ of the subsets $\A_x^T\subseteq \A_x^{\infty:0}$, defined by $\chi_{\A_x^{T}}(\xb)=0$ only if $\xb \in \A_x^{T}$ and $\infty$ otherwise.  These functions allow us to express the constrained optimization problem~\eqref{eq:RT-control-problem-def} as an unconstrained optimization problem on all of $\A_x^{\infty:0}$ through
\begin{equation}
   \min_{\xb \in \A_x^{\infty:0}} \mcR_T(\xb) + \chi_{\A_x^T}(\xb).
    \label{lab_unconstrained_optimality}
\end{equation}
Thus, Theorem~\ref{thm:asym-consistency} can be stated concisely as a guarantee that the $\argmin$ and $\lim$ operations can be interchanged regret-optimization problem's time-horizon becomes unbounded. Hence, we seek to show that
\begin{equation}
\label{eq_Gamma_summary}
    \lim\limits_{T \uparrow \infty} \argmin_{\xb \in \A_x^{\infty:0}} \mcR_T(\xb) + \chi_{\A_x^T}(\xb)
    \in 
\argmin_{\xb \in \A_x^{\infty:0}} 
\lim\limits_{T \uparrow \infty} \mcR_T(\xb) + \chi_{\A_x^T}(\xb)
.
\end{equation}

As with the proof of Theorem~\ref{thm:finite-existence}, our approach is to first establish the relevant properties of the sequence of  functionals~\eqref{lab_unconstrained_optimality} through various lemmas.  These properties include lower semi-continuity and lower-boundedness on $\A_x^{\infty:0}$, their $\Gamma$-convergence to $\mcR_{\infty}$, and their equi-coercivity.  The proof of the aforementioned results then follows from the Fundamental Theorem of $\Gamma$-convergence (\citep[Theorem 2.1]{BraidesLocalMinimization}) which guarantees that~\eqref{eq_Gamma_summary} holds.  

\subsubsection{{Auxiliary Lemmas for Theorems~\ref{thm:asym-existence} and~\ref{thm:asym-consistency}}}\label{sss_PRoof_Lemmas_Asym_consistence}
\begin{lemma}\label{lem_gamma_convergence_FT_unregularized_regret}
Under Assumption~\ref{ass:f-basic-assumptions}, the sequence $\{F_T+\Phi_T\}_{T\in \N}$ $\Gamma$-converges to $\mcR_{\infty}$ on $\A_x^{\infty:\alpha}$. Moreover, $\{F_T+\Phi_T\}_{T\in \N}$ also $\Gamma$-converges to $\mcR_{\infty}$ on $\A_x^{\infty:0}$ in the weak topology.  
\end{lemma}
\begin{proof}
First note that since $f$ is bounded below, by the minimum values $f^{\star}\in \R$, then $f(x)-f^{\star}\geq 0$ for every $x \in \X$.  By Assumption~\ref{ass:phi-assumptions-general} $\phi\geq 0$ for every $x \in \X$.  Thus, for each $T \in \N$ and each $\xb \in \A_x^{\infty}$, we that
\begin{equation}
    (F_T+\Phi_T)(\xb) = \sum_{t=1}^T f(x_t)- f^{\star} + \phi(\Delta x_t) \geq \sum_{t=1}^T 0 + 0 =0;
    \label{eq_simple_lower_bound_on_F}
\end{equation}
hence, $(F_T+\Phi_T)(\xb)$ is non-negatively valued.  
Next, observe that for any $\xb \in \A_x^{\infty}$ the sequence of real-numbers $\{F_T+\Phi_T(\xb)\}_{T \in \N}$ is non-negative monotonically increasing as a function of $T$.  Hence, the Monotone Convergence Theorem implies that, for each $T \in \N$ and each $\xb \in \A_x^{\infty}$, $\{(F_T+\Phi_T)(\xb)\}_{T \in \N}$ converges point-wise to $(F_{\infty}+\Phi_{\infty})(\xb) = \mcR_{\infty}(\xb)$; moreover the convergence is monotone.  In particular, the sequence of functionals $\{F_T+\Phi_T\}_{T \in \N}$ is monotonically increasing and converges point-wise to $\mcR_{\infty}$ on all of $\A_x^{\infty}$.  

We may therefore apply \citep[Proposition 5.4]{DalMasoGamma1993} to conclude that the lower semi-continuous relaxation $\mcR_{\infty}^{lsc}$ of $\mcR_{\infty}$ on $\A_x^{\infty:\alpha}$ (resp. on $\A_x^{\infty:0}$ for the weak topology) and therefore is the $\Gamma$-limit of $\{F_T+\Phi_T\}_{T \in \N}$ on $\A_x^{\infty:\alpha}$ (resp. on $\A_x^{\infty:0}$ for the weak topology).  Lemma~\ref{lem_lscty_f} guarantees that $F_{\infty}$ is weakly lower semi-continuous and Lemma~\ref{lem_assumption_translation} guarantees that $\Phi_{\infty}$ is lower semi-continuous on $\A_x^{\infty:\alpha}$ (resp. weakly lower semi-continuous on $\A_x^{\infty:0}$).  Since $F_{\infty},\Phi_{\infty}>-\infty$ then the sum $R_{\infty} = F_{\infty}+\Phi_{\infty}$ is well-defined on all of $\X$ and therefore \citep[Proposition 1.9]{DalMasoGamma1993} implies that $\R_{\infty}$ is itself lower semi-continuous 
(resp. weakly lower semi-continuous on $\A_x^{\infty:0}$)
.  Hence, $\mcR_{\infty}^{lsc}=\mcR_{\infty}$ and therefore $\mcR_{\infty}$ is the $\Gamma$-limit of $\{F_T+\Phi_T\}_{T \in \N}$ on $\A_x^{\infty:\alpha}$ (resp. on $\A_x^{\infty:0}$ in the weak topology).  
\end{proof}

We continue our analysis by exhibiting some helpful properties of these indicator functions.  
\begin{lemma}[Regularity of Indicator Functions {$\chi_{\A_x^T}$}]\label{lem_wk_lwr_semi_continuity}
For any $T \in \N$, $\A_x^T$ is convex and weakly closed in $\A_x^{\infty:0}$.  
Thus, the function $\chi_{\A_x^T}:\A_x^{\infty:0}\rightarrow \R \cup \{\infty\}$ is weakly lower semi-continuous and convex.  In particular, the function $\chi_{\A_x^T}:\A_x^{\infty:0}\rightarrow \R \cup \{\infty\}$ is lower semi-continuous and convex on $\A_x^{\infty:\alpha}$.  
\end{lemma}
\begin{proof}
    Fix $k \in \R$ and 
    $\xb,\yb \in \A_x^T$.  By linearity of $\Delta$ we have that $\Delta(\xb+k\yb)_u=0$ for every $u\geq T$; 
    thus, $\A_x^T$ is a linear space and it is therefore convex.  Thus, $\chi_{\A_x^T}$ is convex; for each $T \in \N$.

    Let $\{\xb^n\}_{n \in \N}$ be a sequence of algorithms in $\A_x^T$ converging to some algorithm $\xb \in \A_x^{\infty:0}$.  Then, $\sum_{t\geq T} \|\Delta x^n_t- \Delta x_t\|^p \to 0$.  However, by the definition of $\A^{T}$, for all $n\in\N$ we have that $x^n_t =0$ if $t\geq T$. Therefore,
    $$
    \sum_{t\geq T} \|\Delta x^n_t - \Delta x_t\|^p = 
    \sum_{t\geq T} \|\Delta x_t\|^p;
    $$
    hence $\sum_{t\geq T} \|\Delta x^n_t- \Delta x_t\|^p \mapsto 0$ only if $\Delta x_t=0$ for every $t\geq T$.  Thus, $\A_x^T$ is a closed convex subset of $\A_x^{\infty:0}$; for every $T\in \N$.  
    Now, by \citep[1.5 Corollary]{ConwayFunctional2} we conclude that $\A_x^T$ is weakly closed.  
    
    Since, for every $T \in \N$, the set $\A_x^T$ is closed and convex, then \citep[Example 3.4]{DalMasoGamma1993} implies that each $\chi_{\A_x^T}$ is lower semi-continuous.  Since it is lower semi-continuous and convex \citep[Proposition 1.18]{DalMasoGamma1993} implies that $\chi_{\A_x^T}$ is also weakly lower semi-continuous.   
 \end{proof}

In what follows, we frequently make use of the notion of \textit{continuous convergence} (see \citep[Chapter 20, Section 6]{KuratowksiTopology1968}).  In general, this mode of continuous is strictly stronger than point-wise convergence but strictly weaker than uniform convergence.  Continuous convergent functionals are abundant enough to exhibit while simultaneously being regular enough to control.  
\begin{definition}[Continuous Convergence]\label{defn_cnt_conv}
    A sequence $\{F^n\}_{n \in \N}$ from a topological space $Z$ to $\R \cup \{\infty\}$ \textit{converge continuously} to $F:Z\rightarrow \R \cup \{\infty\}$ if, for every $z \in Z$ and every open neighbourhood $U_{F(z)}\subseteq \R \cup \{\infty\}$ of $F(z)$, there exists some $N_z \in \N$ and some open neighbourhood $U_z\subseteq Z$ of $z$ for which
    $$
    F_n(y) \in V_{F(z)} 
    ,
    $$
    for every $n\leq N_z$ and every $y \in U_z$.  
\end{definition}

\begin{lemma}\label{lem_continuous_convergence}
The sequence of functionals $\chi_{\A_x^T}$ converge continuously to the constant-zero functional on $\X$ mapping any $x \in \X$ to $0$.  
\end{lemma}
\begin{proof}
    If $\xb \in \A_x^T$ then $x_t=0$ for every $t\geq T+1>T$ and therefore, $\xb \in \A_x^{T+1}$; thus, $\A_x^T\subseteq \A_x^{T+1}$.  Moreover, since the algorithm $\xb^{T:T+1}$ defined by $\xb^{T:T+1}_{T+1}=x$ is in $\A_x^{T+1}-\A_x^T$ then $\{\A_x^T\}_{T \in \N}$ is a strictly nested increasing sequence of closed linear subsets of $\A_x^{\infty:0}$.  Therefore, $\left\{\chi_{\A_x^T}\right\}_{T \in \N}$ converges point-wise and in a strictly monotonically decreasing fashion to the functional $\chi_{\cup_{T \in \N} \A_x^T}$.  
    Hence, by \citep[Proposition 5.7]{DalMasoGamma1993} the lower semi-continuous relaxation $\chi_{\cup_{T \in \N} \A_x^T}^{lsc}$ of $\chi_{\cup_{T \in \N} \A_x^T}$ is the $\Gamma$-limit of the sequence of functionals $\{\chi_{ \A_x^T}\}_{T \in \N}$.  
    
    Similarly, $\{-\chi_{ \A_x^T}\}_{T \in \N}$ is strictly monotonically increasing with limit point-wise limit the functional $-\chi_{\cup_{T \in \N} \A_x^T}$.  Hence, by \citep[Proposition 5.4]{DalMasoGamma1993} the sequence $\{-\chi_{ \A_x^T}\}_{T \in \N}$ $\Gamma$-converges to $-\chi_{\cup_{T \in \N} \A_x^T}^{lsc}$.  Thus, by definition the sequence $\{-\chi_{\A_x^T}\}_{T \in \N}$ converges continuously to $\chi_{\cup_{T \in \N} \A_x^T}^{lsc}$.  
    
    It remains to compute $\chi_{\cup_{T \in \N} \A_x^T}^{lsc}$.  By \citep[Example 3.4]{DalMasoGamma1993}
    \begin{equation}
        \chi_{\cup_{T \in \N} \A_x^T}^{lsc} = \chi_{\overline{\cup_{T \in \N} \A_x^T}},
        \label{eq_lsc_relaxation_limit_computation}
    \end{equation}
    where $\overline{\cup_{T \in \N} \A_x^T}$ denotes the closure of $\cup_{T \in \N} \A_x^T$ in $\A_x^{\infty:0}$.  We show that this closure is all of $\A_x^{\infty:0}$, in other words, we show that $\cup_{T \in \N} \A_x^T$ is dense in $\A_x^{\infty}$.  First, observe that $\yb \in \cup_{T \in \N} \A_x^T$ only if there exists some $T \in \N$ for which $\Delta y_t=0$ for all $t\geq T$.  Indeed, if $\xb \in \A_x^{\infty:0}$ then , by definition, given any $\epsilon>0$ the exists some $T_{\epsilon}\in \N$ for satisfying
    \begin{equation}
        \sum_{t\geq T_{\epsilon}} \left\|\Delta x_t\right\|^p < \epsilon.
        \label{eq_truncation_estimate}
    \end{equation}
    Observe that, the sequence $\xb^{\epsilon}$ defined by 
    $$
    x^{\epsilon}:= 
    \begin{cases}
    x_t &: t\leq T_{\epsilon}\\
    0 & : t> T_{\epsilon}
    \end{cases}
    $$
    belongs to $\cup_{T \in \N} \A_x^T$.  In particular, the tail estimate~\eqref{eq_truncation_estimate} implies that
    $$
    \sum_{t\geq T_{\epsilon}} \left\|\Delta x_t^{\epsilon} -\Delta x_t\right\|^p
    =
    \sum_{t\geq T_{\epsilon}} \left\|\Delta x_t\right\|^p < \epsilon.
    $$
    Thus, any $\xb \in \A_x^{\infty:0}$ an an accumulation point of some sequence in $\bigcup_{T \in \N} \A_x^{\infty:0}$, for the strong topology.  Since $\cup_{T \in \N} \A_x^{T}$ is the union of a nested sequence of convex sets it is itself convex and since the closure of a convex set is itself convex then $\overline{\bigcup_{T \in \N} \A_x^T}$ is convex.  
    Note that, \citep[1.5 Corollary]{ConwayFunctional2} implies that the weak closure of $\overline{\bigcup_{T \in \N} \A_x^T}$ equals to $\A_x^{\infty:0}$.  Therefore, our remaining computations hold both in the weak and strong topologies on $\A_x^{\infty:0}$ (and consequentially also on the subsets $\A_x^{\infty:\alpha}$).  
    
    Therefore,~\eqref{eq_lsc_relaxation_limit_computation} simplifies to
    $
        \chi_{\cup_{T \in \N} \A_x^T}^{lsc} = \chi_{\overline{\cup_{T \in \N} \A_x^T}} = \chi_{\A_x^{\infty:0}} = 0
        .
        $
\end{proof}
We are now in place to take the first main step to proving Theorems~\ref{thm:asym-existence} and~\ref{thm:asym-consistency}.  Namely, we are in place to conclude that $\mcR_{\infty}$ is the $\Gamma$-limit of the sequence of functionals $\{\mcR_T + \chi_{\A_x^T}\}_{T \in \N}$.
\begin{lemma}\label{lem_Gamma_convergence_done}
The $\mcR_{\infty}$ is the $\Gamma$-limit of $\{\mcR_T + \chi_{\A_x^T}\}_{T \in \N}$ on $\A_x^{\infty:\alpha}$ and
$\mcR_{\infty}$ is also the $\Gamma$-limit of $\{\mcR_T + \chi_{\A_x^T}\}_{T \in \N}$ on $\A_x^{\infty:0}$.
\end{lemma}
\begin{proof}
    By Lemma~\ref{lem_gamma_convergence_FT_unregularized_regret}, $\{F_T + \Phi_T\}_{T \in \N}$ $\Gamma$-converges to $\mcR_{\infty}$ on $\A_x^{\infty:\alpha}$ (resp. on $\A_x^{\infty:0}$ for the weak topology).  By Lemma~\ref{lem_continuous_convergence}, $\{\chi_{A_x^T}\}_{T \in \N}$ converges to the constant $0$ functional continuously.  Moreover, since $0$ is everywhere finite on $\A_x^{\infty:0}$ then \citep[Proposition 6.20]{DalMasoGamma1993} implies that 
    \begin{equation}
        \Gamma-\lim\limits_{T \uparrow \infty} \mcR_T + \chi_{\A_x^T} = 
        \sup_{T \in \N} \mcR_T^{lsc} + 0
        = 
        \sup_{T \in \N} \mcR_T^{lsc} 
        \label{eq_sup_form}
        ;
    \end{equation}
    (where we take the lower semi-continuous relaxation with respect to the strong topology on $\A_x^{\infty:\alpha}$, otherwise, mutatis mutandis, we take it with respect to the weak topology on $\A_x^{\infty:0}$.  
    
    By Lemmas~\ref{lem_lscty_f},~\ref{lem_assumption_translation}, and \citep[Proposition 1.9]{DalMasoGamma1993} we find that each $\mcR_T$ is lower semi-continuous.  Thus, the right-hand side of~\eqref{eq_sup_form} yields equal to $\sup_{T \in \N} \mcR_T^{lsc} $.  Hence, the conclusion follows since
    $
       \Gamma-\lim\limits_{T \uparrow \infty} \mcR_T + \chi_{\A_x^T} = 
        \sup_{T \in \N} \mcR_T^{lsc} = \mcR_{\infty}
    .
    $
 \end{proof}
\begin{lemma}[Equi-coercivity Lemma]\label{lem_equi_coercivity}
The family of functionals $\left\{
\mcR_T + \chi_{\A_x^T}
\right\}_{T\in \N}$ is equi-coercive on $\A_x^{\infty:\alpha}$. In addition, $\phi$ the family of functionals $\left\{
\mcR_T + \chi_{\A_x^T}
\right\}_{T\in \N}$ is equi-coercive on $\A_x^{\infty:0}$ for the weak topology.  
\end{lemma}
\begin{proof}
    Fix $T \in \N$ and $\xb \in \A_x^{\infty:0}$.  If $\xb \in \A_x^T$ then by definition $\chi_{\A_x^T}(\xb)=0$ and $\Delta x_t =0$ for every $t\geq T$.  
    Since Assumption~\ref{ass:phi-assumptions-general} guarantees that $\phi(0)=0$, then $\phi(\Delta x_t)=0$ for every $t \geq T$.  Therefore, we may compute:
    \begin{equation}
        \begin{aligned}
        \left(\mcR_T + \chi_{\A_x^T}\right)(\xb) & = \sum_{t=1}^T f(x_t) - f^{\star} + \phi_t(\Delta x_t) + \chi_{\A_x^T}(\xb)\\
        & = \sum_{t=1}^T f(x_t) - f^{\star} + \sum_{t=1}^T \phi_t(\Delta x_t) \\
        & = \sum_{t=1}^T f(x_t) - f^{\star} + \sum_{t=1}^T \phi_t(\Delta x_t)  +0 \\
        & = \sum_{t=1}^T f(x_t) - f^{\star} + \sum_{t=1}^{\infty} \phi_t(\Delta x_t)\\
        & \geq \sum_{t=1}^{\infty} \phi_t(\Delta x_t)
        ;
        \end{aligned}
        \label{proof_lem_equi_coercivity_eq_1}
    \end{equation}
    where the last inequality in~\eqref{proof_lem_equi_coercivity_eq_1} holds since $-\infty < f^{\star}\leq f(x)$ by Assumption~\ref{ass:f-basic-assumptions}.  
    
    If $\xb \not \in \A_x^T$, then $\chi_{\A_x^T}(\xb)=\infty$.  In which case we necessarily have 
    \begin{equation}
    \left(\mcR_T + \chi_{\A_x^T}\right)(\xb) = \infty\geq 
    \sum_{t=1}^{\infty} \phi_t(\Delta x_t).
        \label{proof_lem_equi_coercivity_eq_2}
    \end{equation}
    Thus, together~\eqref{proof_lem_equi_coercivity_eq_1} and~\eqref{proof_lem_equi_coercivity_eq_2} imply that for every $T \in \N$ and every $\xb \in \A_x^{\infty:0}$ the following bound must hold:
    \begin{equation}
    \left(\mcR_T + \chi_{\A_x^T}\right)(\xb) \geq 
    \sum_{t=1}^{\infty} \phi_t(\Delta x_t) = \Phi_{\infty}(\xb).
        \label{proof_lem_equi_coercivity_eq_3_general_bound_over_Space}
    \end{equation}
    
    Therefore, by Lemma~\ref{lem_assumption_translation} and \citep[Proposition 7.7]{DalMasoGamma1993}, we may conclude that \\
    $\left\{\mcR_T + \chi_{\A_x^T}\right\}_{T\in \N}$ forms an equi-coercive family on $\A_x^{\infty:\alpha}$ (resp. on $\A_x^{\infty:0}$ for the weak topology if $\phi$ is also convex) of functionals since $\Phi_{\infty}$ is itself coercive on $\A_x^{\infty:\alpha}$ (resp. on $\A_x^{\infty:0}$ for the weak topology).  
\end{proof}
%
\begin{lemma} \label{lem:mcP0-mcP-equivalence-nonempty}
    Suppose that $\mcP_x^{\infty:0}\neq\emptyset$. Then $\mcP_x^\infty = \mcP_x^{\infty:0}$, and hence $\mcP_x^\infty$ is also non-empty.
\end{lemma}
\begin{proof}
    The proof follows from the fact that $\mcR_{\infty}$ is bounded over both $\mcP_x^{\infty}$ and $\mcP_x^{\infty}$.
    
    Let us separate $\A_x^\infty = \mcP_x^{\infty:0} \cup C_x^\infty \cup  D_x^\infty$ into the three disjoint parts,
    where we define $C_x^\infty = \A_x^{\infty:0} \setminus \mcP_x^{\infty:0}$ and $ D_x^\infty = \A_x^\infty \setminus \A_x^{\infty:0}$. We show that the claim of the lemma holds by demonstrating the equivalent claim that for any $\xb \in \mcP_x^{\infty:0}$ and $\yb\in C_x^\infty \cup  D_x^\infty$, we have that $\mcR_{\infty}(\xb) < \mcR_{\infty}(\yb)$. Note that the above claim is equivalent to the claim of the lemma since it implies that for any $\yb\in\A_x^\infty$, we have $\mcR_\infty(\xb) \leq \mcR_\infty(\yb)$ if and only if $\xb\in\mcP_x^{\infty:0}$.
    
    First, we note that by the definition of $\mcP_x^{\infty:0}$, since $C_x^\infty \subseteq \A_x^{\infty,0}$, we have that for any $\xb \in \mcP_x^\infty$ and $\yb\in C_x^\infty$, we have that $\mcR_{\infty}(\xb) < \mcR_{\infty}(\yb)$.
    
    Next, we begin by recalling that if $\xb\in \mcP_x^{\infty:0}$, then $\mcR(\xb)<\infty$. This holds since by defining $\yb\in \A_x^{\infty:0}$ such that $y_0 = x$ and $y_u = \xst$ for all $u>0$, we have that $\mcR_\infty(\yb) = \phi(x-\xst) < \infty$, which by the definition of $\xb\in \mcP_x^{\infty:0}$ implies that $\infty > \mcR_{\infty}(\yb) \geq \mcR_{\infty}(\xb) \geq 0$.
    By the definition of $\A_x^{\infty:0}$, we have that for any $\yb\in D_x^\infty$, the sum $\sum_{u=0}^\infty \| \Delta y_u \|^p = \infty$ diverges. Moreover, by Assumption~\ref{ass:phi-assumptions-general}, we have that there exists a constant $c>0$ such that $\phi(z) \geq c \| \cdot \|^p$. Hence, 
    \begin{equation*}
        \yb \in D_x^\infty
        \;\Longrightarrow \;
        \mcR_\infty(\yb) \geq \sum_{u=0}^\infty \phi(\Delta y_u) \geq c \sum_{u=0}^\infty \| \Delta y_u \|^p = \infty
        \;.
    \end{equation*}
    We have shown that $\mcR_\infty(\xb)<\infty$ for all $\xb\in\mcP_x^{\infty:0}$ and $\mcR_\infty(\yb)=\infty$ for all $\yb \in D_x^\infty$. Combining these facts, we obtain that $\mcR_{\infty}(\xb) < \mcR_{\infty}(\yb)$, concluding the proof.
\end{proof}
\subsubsection{{Proof of Theorem~\ref{thm:asym-existence}}}

\label{sss_proof_asym_existence_with_lemmas_ready}
We are now in place to prove Theorem~\ref{thm:asym-existence}.  Since the proof of the  $\A_x^{\infty,0}$ case and the $\A_x^{\infty,\alpha}$ (for $\alpha>0$) case are analogous, we simplify our exposition by combining them and highlight their differences when necessary.

\begin{proof}
    By Lemma~\ref{lem_Gamma_convergence_done} and~\ref{lem_equi_coercivity} the family of functionals $\left\{
\mcR_T + \chi_{\A_x^T}
\right\}_{T\in \N}$ is lower semi-continuous and equi-coercive on $\A_x^{\infty:\alpha}$ (resp. weakly lower semi-continuous and equi-coercive with respect to its weak topology on $\A_x^{\infty:0}$).  Therefore, \citep[Theorem 7.8]{DalMasoGamma1993} implies that $\mcR_{\infty}$ is coercive on $\A_x^{\infty:\alpha}$ (on $\A_x^{\infty:0}$ with respect to the weak topology).  

Lemma~\ref{lem_bound_coercive} guaranteed that, for every $T \in \N$, the regret functionals $\mcR_T$ are all coercive on $\A_x^{\infty:\alpha}$ (resp. on $\A_x^{\infty:0}$ with respect to the weak topology).  Since $\mcR_{\infty} = \sup_{T \in \N} \mcR_T$ then \citep[Proposition 1.8]{DalMasoGamma1993} guarantees that $\mcR_{\infty}$ is lower semi-continuous on $\A_x^{\infty:\alpha}$ (resp. on $\A_x^{\infty:0}$ with respect to the weak topology).  

Next, Lemma~\ref{lem_bound_coercive} guaranteed that, for every $T \in \N$, the regret functional $\mcR_T$ takes non-negative values.  Hence, for every $\xb \in \A_x^{\infty}$ we compute
\begin{equation}
    \mcR_{\infty}(\xb) = \sup_{T \in \N} \mcR_T(\xb) \geq 0.
\end{equation}
Therefore, $\mcR_{\infty}$ is both lower semi-continuous and coercive on $\A_x^{\infty:\alpha}$ (resp. weakly lower semi-continuous and coercive with respect to the weak topology on $\A_x^{\infty:0}$) and it is bounded below by $0$.  Hence, \citep[Theorem 1.15]{DalMasoGamma1993} implies that $\mcP^{\infty:\alpha}_x\neq \emptyset$ ($\mcP^{\infty}_x\neq \emptyset$).  

Now, Assumptions~\ref{ass:f-basic-assumptions} and~\eqref{ass:phi-assumptions-general} imply that the sequence $\xb^{\infty}$ defined by $x_0^{\infty}=x$ and $x_t^{\infty}=x^{\star}$ for $t\geq 1$ satisfies $\xb \in \A_x^{\infty:\alpha}$ (resp. in $\A_x^{\infty:0}$) and $\mcR_{\infty}(\xb)<\infty$.  Hence, by definition, any $\xb^{\star} \in \mcP^{\infty:\alpha}_x$ (resp. in $\mcP^{\infty:0}_x$) must satisfy
$
\mcR_{\infty}(\xb^{\star})\leq \mcR_{\infty}(\xb^{\infty})<\infty
.
$
Therefore, $\mcP^{\infty:\alpha}_x$ (resp. $\mcP^{\infty:0}_x$ is non-empty and any algorithm therein has finite regret.

Lastly, we apply Lemma~\ref{lem:mcP0-mcP-equivalence-nonempty} to conclude that $\xb\in\mcP^{\infty:0}\;\Rightarrow\;\xb\in\mcP^{\infty}$ and hence that $\mcP^{\infty}$ is non-empty. 
\end{proof}

\subsection{Proof of Corollary~\ref{cor:convergence-asym-nonconv}}
\begin{proof}
    Since $\xb\in\mcP$, Theorem~\eqref{thm:asym-existence} implies that
    \begin{equation*}
       \mcR_{\infty}(\xb) = \sum_{t=0}^\infty f(x_{t}) - f^{\star} + \phi(\Delta x_t)
       =
       \sum_{t=0}^\infty a_t
       <\infty
       \;,
    \end{equation*}
    where the summand $a_t = f(x_{t})-f^{\star} + \phi(\Delta x_t) $ is non-negative. Hence, we have that $\lim_t a_t  = 0$. 
    
    Now, assume that the sequence is monotone. Since the $a_t$ are summable, the partial sums $S_t = \sum_{u=t}^\infty a_t \rightarrow 0$ asymptotically vanish. Hence, we have the bound
    \begin{align*}
        2 S_t \geq 2 \sum_{u=t}^{2t} a_u \geq  2 t a_{2 t} \geq 0
        \,.
    \end{align*}
    We therefore have that $\lim_{t\rightarrow\infty} t \, a_t = 0$, proving the claim of the theorem.
    
\end{proof}

\subsection{Proof of Theorem~\ref{thm:asym-consistency}}
The Lemmas established in~\ref{sss_PRoof_Lemmas_Asym_consistence} reduce the proof of Theorem~\ref{thm:asym-consistency} to a simple consequence of the Fundamental Theorem of $\Gamma$-convergence.  As with the proof of Theorem~\ref{thm:asym-existence}, since the proof of the convex and the non-convex cases are analogous, mutatis mutandis, and are therefore combined.  
\begin{proof}
    Since $\mcR_{\infty}$ is the $\Gamma$-limit of the equi-coercive sequence of functionals $\{\mcR_{T}+\chi_{\A_x^T}\}_{T \in \N}$ on $\A_x^{\infty:\alpha}$ (resp. on $\A_x^{\infty:0}$ with respect to the weak topology), then result therefore follows directly from \citep[Corollary 2.1]{BraidesLocalMinimization}.
\end{proof}

\newpage

\section{Proofs for Section~\ref{sec:first-order-conditions-smth}}

\subsection{Proof of Theorem~\ref{thm:gateaux-derivative}}
\begin{proof}
    We begin by recalling the definition of the Gâteaux derivative,
    \begin{equation*}
    \mcR_T^\prime(\xb)(\delxb)
    =\lim_{\epsilon\rightarrow 0} \frac{\mcR_T(\xb + \epsilon \delxb ) - \mcR_T(\xb)}{\epsilon}
    \;,
    \end{equation*}
    where $\delxb = \yb-\xb$ for some $\yb\in\A_x^T$. Expanding the definition of $\mcR_T$, exchanging the limit with the sum and applying the assumed smoothness of Assumption~\ref{ass:f-phi-differntiable}, we obtain that
    \begin{equation*}
        \mcR_T^\prime(\xb)(\delxb)
        =
        \sum_{t=1}^T
        \langle \nabla f(x_t) \mathrel{,} \delx_t \rangle
        + \langle \nabla \phi( \Delta x_{t-1} ) \mathrel{,} \Delta \delx_{t-1} \rangle
        \;.
    \end{equation*}
    Lastly, noting that $\delx_{0}=0$ since $x_0=y_0=x$ and re-arranging the sum, we obtain the expression in the statement of the theorem.
\end{proof}

\subsection{Proof of Theorem~\ref{thm:optimal-dynamics}}
\begin{proof}
    First we show that $\mcPh_x^T$ is precisely the set of $\xb\in\A_x^T$ which make $\mcR_T^\prime$ vanish. To see this, note that $\mcR^\prime(\xb)(\delxb)$ given in equation~\eqref{eq:gateaux-expression} is a bounded linear functional in $\delxb$, and hence vanishes if and only if we have
    \begin{equation*}
        \nabla f(x_{t+1}) - \nabla \phi(x_t) - \nabla \phi(x_{t+1}) = 0
    \end{equation*}
    for all $t=0,1,\dots,T-1$. Hence, we have that
    \begin{equation*}
    \mcP_x^T = 
    \{ \xb\in\A_x^T : \mcR^\prime_T(\xb) \equiv 0 \}
    \;,
    \end{equation*}
    and is by definition the set of critical points. Since $\A_x^T$ is open and since $\mcP_x^T\subseteq \A_x^T$ is non-empty, we must have that $\mcR^\prime_T(\xb) = 0$ for any $\xb\in\mcP_x^T$. Hence, we obtain the inclusion $\mcPh_x^T\supseteq\mcP_x^T$. Lastly, it is easy to see that the recursion-\emph{ii} holds by the stationarity of equation~\eqref{thm:optimal-dynamics}.
    
\end{proof}

\subsection{Proof of Theorem~\ref{thm:pseudo-gd-nonconvex-2}}
\begin{proof}
	We begin by the case where $T<\infty$. Recall that by the dynamic programming principle (Lemma~\ref{lem:dynamic-programming-principle-fin}) that
	\begin{equation*}
		J^{T-t}(x_t) = \min_{y\in\X} \{ \phi(y-x_t) + f(y) + J^{T-(t+1)}(y) \}
		\;,
	\end{equation*}
	where $x_{t+1}\in C_x^{T-t} = \argmin_{y\in\X} \{ \phi(y-x_t) + f(y) + J^{T-(t+1)}(y) \}$. Since $x_{t+1}\in C_x^{T-t}$, by the differentiability of $f$ and $\phi$, as well as the assumed local-Lipschitz property of $J^{T-(t+1)}$, we find that 
	\begin{equation*}
		0 \in \nabla \phi(x_{t+1}-x_t) + \nabla f(x_{t+1}) + \partial J^{T-(t+1)}(x_{t+1})
		\;,
	\end{equation*}
	and hence that
	\begin{equation} \label{eq:proof-gen-dynamics-finite-1}
		- \nabla \phi(x_{t+1}-x_t) -  \nabla f(x_{t+1}) \in \partial J^{T-(t+1)}(x_{t+1})
		\;.
	\end{equation}
	Since $\xb\in\mcPh^T$, we have that the optimal dynamics of equation~\eqref{eq:optim-diff-eq} must hold and hence we get that $\nabla \phi(x_{t+1}-x_t) + \nabla f(x_{t+1}) = \nabla \phi(x_{t+2}-x_{t+1})$, yielding the claim of the theorem.

	In the case of $T=\infty$, we can at arrive at equation~\eqref{eq:proof-gen-dynamics-finite-1} with the DPP (Lemma~\ref{lem:dynamic-programming-principle-asym}) and applying the same sequence of steps, which yield that for all $t\in\N$,
	\begin{equation} \label{eq:proof-gen-dynamics-asym-1}
		- \nabla \phi( \Delta x_t ) -  \nabla f(x_{t+1}) \in \partial \Jinf(x_{t+1})
		\;,
	\end{equation}
	where the assumed local Lipschitzness of $\Jinf$ ensures that $\partial \Jinf$ is always non-empty. 
%
%

    In order to show that the recursion, applying the DPP of Lemma~\ref{lem:dynamic-programming-principle-asym} twice implies that for all $t$
    \;,
    \begin{align*}
        \Jinf(x_t) = \phi(\Delta x_t) + f(x_{t+1}) + \phi(\Delta x_{t+1}) + f(x_{t+2}) + \Jinf(x_{t+1})
        \;.
    \end{align*}
    and that $(x_{t+1},x_{t+2})\in\argmin_{y,z\in\X}\left\{ \phi(y-x_t) +\phi(z-y) +  f(y) + f(z) + \Jinf(z) \right\}\textbf{}\;.$
    By the local Lipschitz smoothness of the above function, we know that its generalized derivative must contain zero at $(x_{t+1},x_{t+2})$. Taking the (generalized) derivative at $y=x_t$ and letting it vanish we obtain that
    \begin{equation} \label{eq:proof-gen-dynamics-asym-2}
        \nabla \phi(\Delta x_t) = 
        \nabla \phi( \Delta x_{t+1}) + \nabla f(x_{t+1})
        \;,
    \end{equation}
    yielding the desired recursion.
	
	Hence, combining with equation~\eqref{eq:proof-gen-dynamics-asym-1}, we find that we must have that
	\begin{equation*}
		- \nabla \phi( \Delta x_{t+1}) \in \partial \Jinf(x_{t+1})
	\end{equation*}
	for all $t$, as well as the recursion~\eqref{eq:proof-gen-dynamics-asym-2}. Noting that $x=x_t$ and $t\in\N$ can be chosen arbitrarily, we have that $\nabla \Jinf(x_0) = - \nabla \phi( \Delta x_0)$ for all $x_0\in\Upsilon$, where
    \begin{equation*}
       \Upsilon_1 = \{ y\in\X \;:\; 
       \exists \xb\in\mcP^\infty \text{ such that } y=x_1
       \} \subseteq \X \;,
    \end{equation*}
    the set of points that can be reached in a single step of an algorithm $\xb\in\mcP^\infty$. 
    
    We now show that $\Upsilon_1=\X$. Using equation~\eqref{eq:proof-gen-dynamics-asym-1} and the property that $\nabla \phi \circ \nabla \phi^\ast = \mathrm{id}$ we have that
    \begin{equation} \label{eq:proof-gen-recursion-asym-3}
        x_0 = x_1 - \nabla \phi^\ast( - \nabla f(x_1) - \nu(x_1) )
        \;,
    \end{equation}
    for some $\nu(x_1)\in\partial \Jinf(x_1)$.
    Hence, for any $x_1=y\in\X$, we can pick $\xb\in\mcP_x^\infty$ where $x=x_0$ defined according to~\eqref{eq:proof-gen-recursion-asym-3}, which shows that $\Upsilon_1\supseteq\X$, as desired. We therefore have that $\nabla \phi(\Delta x_t) = - \nabla J^{T-1}(x_t)$ for all $\xb\in\mcP^T$. Lastly, it is easy to see that the recursion follows from~\eqref{eq:proof-gen-dynamics-asym-2}.
\end{proof}

\newpage

\section{Proofs for Section~\ref{sec:convex-optimization}}

Over the course of this section, we assume without loss of generality that $\fst = 0$ since we may simply consider the function $\tilde{f}(x) = f(x) - \fst$, which satisfies this property.

\subsection{Proof of Lemma~\ref{lem:convex-uniqueness}}
\begin{proof}
    Consider $\xb,\yb\in\A_x^T$ and $\rho\in(0,1)$. Then we have
    \begin{align*}
        \mcR_T(\xb + \rho (\yb-\xb) )
        &= \sum_{t=1}^T f( x_t + \rho (y_t - x_t) ) + \phi(\Delta x_t + \rho (\Delta x_t - \Delta y_t) )
        \;.
    \end{align*}
    Noting that by the convexity of $f$ and the strict convexity of $\phi$, we have
    \begin{align*}
        f( x_t + \rho (y_t - x_t) ) &\leq
        (1-\rho) f( x_t ) + \rho f( y_t )
        \;, \\
        \phi(\Delta x_t + \rho (\Delta x_t - \Delta y_t) )
        &<
        (1-\rho)\phi(\Delta x_t ) + \rho\phi(\Delta y_t )
        \;,
    \end{align*}
    and hence, we find that 
    \begin{equation*}
         \mcR_T(\xb + \rho (\yb-\xb) )
         < (1-\rho)  \mcR_T(\xb) + \rho  \mcR_T(\yb)
         \;,
    \end{equation*}
    showing that $\mcR_T$ is strictly convex, and hence has a unique minimum. Applying~\cite[Proposition 1.2]{ekeland1999convex}, the solution must be unique and by~\cite[Proposition 2.1]{ekeland1999convex} is the unique critical point of $\mcR_T$.
\end{proof}

\subsection{Proof of Lemma~\ref{lem:convex-J-properties-1}}

\begin{proof}
    Let us first assume that $T\in\N\cup\{\infty\}$. In order to show that $J^T$ is convex and differentiable we leverage convex analysis tools from~\cite{rockafellar1970convex}. Let us introduce the (abuse of) notation
    \begin{equation*}
        \mcR_T(x;\xb_{1:T}) = \mcR_T(\xb)
        \;,
    \end{equation*}
    for $\xb\in \A^T_x$ such that $x_0 = x$ and $\xb_{1:T}=\{x_t\}_{t=1}^T$, which allows us to separate the initial value and the remainder of the path of the optimizer. Note that by the convexity of $f$ and $\phi$ that $\mcR_T$ is convex in both variables and that by definition we have $J^T(x) = \min_{\yb_{1:T}\in\X^{\otimes T}}\mcR_T(x;\yb_{1:T})$. Now for $x,y\in\X$, $\xb_{1:T},\yb_{1:T}\in\X^{\otimes T}$, $\rho\in(0,1)$
    \begin{align*}
        J^T( (1-\rho) \, x + \rho \, y )
        &= \min_{\zb_{1:T}\in\X^{\otimes T}} \mcR_T( \, (1-\rho) \, x + \rho \, y \,;\zb_{1:T})
        \\ &\leq 
        \mcR_T \left( \, (1-\rho) \, x + \rho \, y \,;\, (1-\rho)\,\xb_{1:T} + \rho \, \yb_{1:T} \, \right)
        \\ &\leq 
        (1-\rho) \, \min_{\xb_{1:T}\in\X^{\otimes T}} \mcR_T(x;\xb_{1:T}) + \rho \, \min_{\yb_{1:T}\in\X^{\otimes T}} \mcR_T(y;\yb_{1:T})
        \\ &= (1-\rho) \, J^T(x) + \rho \, J^T(y)
        \;.
    \end{align*}
    Taking the minimum over $\xb_{1:t}$ and $\yb_{1:t}$, we obtain
    \begin{align*}
     J^T( (1-\rho) \, x + \rho \, y )
        &\leq
        (1-\rho) \, \min_{\xb_{1:T}\in\X^{\otimes T}} \mcR_T(x;\xb_{1:T}) + \rho \, \min_{\yb_{1:T}\in\X^{\otimes T}} \mcR_T(y;\yb_{1:T}) 
        \\&= (1-\rho) \, J^T(x) + \rho \, J^T(y)
        \;,
    \end{align*}
    demonstrating the claim that $J^T$ is convex for all $T\in\N\cup\{\infty\}$.
    
    Next, we show that $J^T$ is differentiable for $T\in\N$. Note that for each fixed $\yb_{1:T}$, that the function $x\mapsto \mcR_T(x;\yb_{1:T})$ is both convex and differentiable, where the differentiability of $x$ follows from the differentiability of $\phi$. Fix $x\in\X$, and define a sequence $\{ \yb^i_{1:T} \}_{i\in\N} \subseteq \X^{\otimes T}$ such that
    \begin{equation*}
        f_i(x) = 
        \mcR_T(x;\yb_{1:T}^i) \longrightarrow \min_{\yb_{1:T} \in \X^{\otimes T}} \mcR_T(x;\yb_{1:T}^i) = J^T(x)
        \;.
    \end{equation*}
    since each $f_i$ is convex and differentiable over $\X$, we can apply~\cite[Theorem 25.7]{rockafellar1970convex} to claim that $\nabla f_i(x) \rightarrow \nabla J^T(x)$, and hence $\nabla J^T(x)$ is differentiable.
    
    Next, we show that both of the convergence statements of Lemma~\ref{lem:convex-J-properties-1} hold, which in turn imply the differentiability for $T=\infty$. Once more, we will leverage the results of~\cite[Theorem 25.7]{rockafellar1970convex}. Notice that for each $T\in\N$, $J^T$ is convex and differentiable. Moreover, note that $J^T$ is pointwise non-decreasing and bounded above due to Theorem~\ref{thm:asym-existence}, and hence $J^T \rightarrow \Jinf$ pointwise. Applying~\cite[Theorem 25.7]{rockafellar1970convex}, we get that these properties imply that $J^T\rightarrow \Jinf$ and $\nabla J^T\rightarrow \nabla \Jinf$ uniformly on compact sets, which also show that $\Jinf$ is differentiable.
\end{proof}

\subsection{Proof of Lemma~\ref{lem:gradJ-recursion}}
\begin{proof}
	We note that Lemma~\ref{lem:gradJ-recursion} is a special case of Theorem~\ref{thm:pseudo-gd-nonconvex-2}, where for any $T\in\N\cup\{ \infty \}$, $J^T$ is convex and differentiable. In this case, we find that the necessary conditions of Theorem~\ref{thm:pseudo-gd-nonconvex-2} are satisfied. Moreover, we have that since $J^T$ is differentiable, $\partial J^T (x) = \{ \nabla J^T(x) \} $. Applying this to the result of Theorem~\ref{thm:pseudo-gd-nonconvex-2}, we obtain the desired result.
\end{proof}

\subsection{Proof of Lemma~\ref{lem:convex-J-properties-2}}
\begin{proof}
    We split the proof according to the individual properties listed in the statement of the Lemma.
    
    \paragraph{Proof of Property i.} 
    
    We recall the result from~\cite[Theorem 26.5]{rockafellar1970convex} which states that a function is Legendre convex if and only if its dual is Legendre convex. Hence, it is sufficient for us to show that $(J^T)^\ast$ is Legendre convex. What remains to be shown are that $(J^T)^\ast$ is convex, differentiable and satisfies the property that $\lim_{\|x\|\rightarrow\infty}\|\nabla J^T(x)\|=\infty$.
    
    We first show that $(J^T)^\ast$ is strictly convex. Recall the recursion on $(J^T)^\ast$ from Lemma~\ref{lem:bregman-recursion},
    \begin{equation} \label{eq:prf-prop-v-eq1}
        (J^T)^\ast(q) = \phit^\ast(q) + (J^{T-1} + f)^\ast (q)
        \;.
    \end{equation}
    Since $\phi$ is Legendre convex, $\phit^\ast$ must also be Legendre convex and hence strictly convex. Since $J^{T-1}$ is convex (property~\emph{i}) and $f$ is strictly convex, we therefore have that $(J^{T-1} + f)^\ast$ is convex. Since $(J^T)^\ast = \phit^\ast + (J^{T-1} + f)^\ast$ is the sum of a strictly convex and a convex function, it is strictly convex and hence $(J^T)^\ast$ is strictly convex.
    
    Next, we show that $(J^T)^\ast$ is differentiable. First note that $\phit^\ast$ is differentiable since it is Legendre convex. Next, recall that $J^{T-1} + f$ is strictly convex, and hence by~\cite[Theorem 26.3]{rockafellar1970convex} $(J^{T-1} + f)^\ast$ is differentiable. Hence, by~\eqref{eq:prf-prop-v-eq1} we have that $(J^T)^\ast$ is strictly convex.
    
    Now, note that by \cite[Lemma 26.7]{rockafellar1970convex}, a convex function $g:\X\rightarrow\R$ satisfies $\lim_{\|x\|\rightarrow\infty}\|\nabla g(x)\|=\infty$ if and only if $g$ is co-finite, that is, $g$ satisfies
    \begin{equation*}
    	\lim_{\lambda \rightarrow \infty} g(\lambda y)/\lambda = \infty
    	\;\; \forall 0 \neq y \in \X \;.
    \end{equation*}

    By Fenchel's inequality, we have that $(J^{T-1}+f)^\ast(q) \geq - J^{T-1}(0)+f(0) = -\alpha >-\infty$ and applying Lemma~\ref{lem:bregman-recursion} with $p=0$, we obtain the bound
    \begin{equation*}
    	(J^T)^\ast(q) = \phi^\ast(q) + (J^{T-1}+f)^\ast(q) \geq \phi^\ast(q) - \alpha
    	\;.
    \end{equation*}
    Since $\phi$ is assumed to be Legendre convex~\cite[Theorem 26.5]{rockafellar1970convex} implies that $\phit^\ast$ is also Legendre convex and hence co-finite. Hence, we have that
    \begin{align*}
    	\lim_{\lambda\rightarrow\infty} \frac{(J^T)^\ast( \lambda y ) }{\lambda}
    	\geq \lim_{\lambda\rightarrow\infty} \frac{\phit^\ast(\lambda y) -\alpha}{\lambda} = \infty
    	\;,
    \end{align*}
    showing that $(J^T)^\ast$ is also co-finite and hence Legendre convex, as desired.

    \paragraph{Proof of Property ii.} Let $\xst \in \argmin_{x\in\X} f(x)$, and consider $\xb\in\cup_{T\in\N}\A_{\xst}^T$ defined by $x_t= \xst$ for all $t\in\N$. Note that under this definition, $0 = \mcR(\xb) \geq \min_{\yb\in\A_x^T}\mcR_T(\yb) = J^T(\xst) \geq \min_{x} J^T(x) \geq 0$. By property \emph{i}, we have that $J^T$ is Lengendre convex and hence strictly convex, so we have that this minimum is unique.

    \paragraph{Proof of Property iii.} Note that by property \emph{i}, $(J^T)^\ast$ is Legendre convex. Hence the relative convex with respect to $\phit^\ast$ follows directly from Lemma~\ref{eq:dual-bregman-property}, since we have
    \begin{align*}
        D_{(J^T)^\ast}(q,p) &= 
        D_{\phit^\ast}(q,p)  + D_{(J^{T-1} + f)^\ast}(q,p) 
        \\&\geq 
        D_{\phit^\ast}(q,p)
    \end{align*}
    where the inequality follows from the positivity of the Bregman divergence.
    Hence, we obtain one of the claims for \emph{Property i} of the theorem. For the second claim, we apply the result of Lemma~\ref{lem:interchange-relative-smoothness} to get the desired result.

\end{proof}

\subsection{Proof of Lemma~\ref{lemma:relative-smoothness-convexity-J}}
\begin{proof} We first note here that the assumption that $\phi$ is quadratic implies that $\phit(x) = \phi(-x) = \phi(x)$.
    Now, assume that $f$ is $\lambda$-relatively-smooth with respect to $\phi$. Let us define the set
    \begin{equation*}
        \Gamma = \{ 
        \gamma \in [0,1] \,:\, D_{\Jinf}(x,y) \leq \gamma D_{\phi^\ast} (x,y)
        \;\;\forall\,x,y\in\X
        \}
        \;,
    \end{equation*}
    as well as its infimum $\gammau = \inf \Gamma$.
    Note that by Theorem~\ref{lem:convex-J-properties-2}-\emph{iii}, $\Gamma$ is non-empty and hence $\gammau$ is well-defined. Furthermore, it is easy to see that $\Gamma$ is closed and bounded and hence compact, therefore $\gammau\in\Gamma$. 
    
    Now since $f$ is $\lambda$-relatively-smooth, by the linearity of the Bregman divergence we obtain that $D_{\Jinf + f}(x,y) \leq (\gammau + \lambda) D_\phi(x,y)$, and hence, applying Lemma~\ref{lem:interchange-relative-smoothness}, we have that
    \begin{equation} \label{eq:J+f-ast-breg-bound}
        D_{(\Jinf + f)^\ast}(p,q) \geq (\gammau + \lambda)^{-1} D_{\phi^\ast}(p,q)
        \;.
    \end{equation}
    
    Recalling Lemma~\ref{lem:bregman-recursion}, and taking the limit as $T\rightarrow\infty$, we have
    \begin{equation*}
        D_{(\Jinf)^\ast}( q , p ) = 
        D_{\phi^\ast}\left( q , p \right) +
        D_{(\Jinf+f)^{\ast}}( q , p )
        \;,
    \end{equation*}
    hence combining with~\eqref{eq:J+f-ast-breg-bound}, we obtain
    \begin{equation} \label{eq:dual-Jinf-ast-recursion}
        D_{(\Jinf)^\ast}( q , p ) \geq (1 + (\gammau + \lambda)^{-1} ) D_{\phi^\ast}\left( q , p \right)
        \;.
    \end{equation}
    Applying Lemma~\ref{lem:interchange-relative-smoothness} once more to~\eqref{eq:dual-Jinf-ast-recursion}, we find that
    \begin{equation*}
        D_{\Jinf}( x,y ) \leq (1 + (\gammau + \lambda)^{-1} )^{-1} D_{\phi^\ast}\left( x,y \right)
        \;,
    \end{equation*}
    and hence we have that $(1 + (\gamma + \lambda)^{-1} )^{-1} \in \Gamma$. By the definition of $\gammau$, however, we have that
    \begin{equation} \label{eq:Jinf-gamma-recursion}
        \gammau \leq (1 + (\gammau + \lambda)^{-1} )^{-1}
        \;.
    \end{equation}
    Noting that equation~\eqref{eq:Jinf-gamma-recursion} can be re-arranged into a quadratic inequality in terms of $\gammau$ and that $\gammau\in[0,1]$, we can solve the inequality to obtain that
    \begin{equation*}
        \gammau
        \leq \frac{1}{2} \left( \sqrt{\lambda^2 + 4\lambda} - \lambda \right)
        \in (0,1)
        \;,
    \end{equation*}
    as desired.
    In order to obtain the converse result, we begin by assuming that $f$ is $\mu$-relatively-convex, and repeat the same sequence of steps with the inequalities reversed and modifying the definition of the set $\Gamma$ and of $\gammau$ accordingly (as a $\sup$).
\end{proof}

\subsection{Proof of Theorem~\ref{thm:convergence-basic-thm}}
\begin{proof}
    Over the course of this proof, we use the short-hand notation $\nabla \Jinf(x_{t}) = \nabla \Jinf_t$.
    
    We begin by noting that Lemma~\ref{lemma:dual-descent-Lemma}-2 and the strict convexity of $\phi$ implies that $\phit^\ast( \nabla \Jinf_{t+1} ) \leq \phit^\ast( \nabla \Jinf_{t} )$, and hence $\{ \phit^\ast( \nabla \Jinf_{t+1} ) \}$ is decreasing. Next, recalling Lemma~\ref{lemma:dual-descent-Lemma}-3, we have that
    \begin{equation} \label{eq:descent-proof-proximal-bound-1}
		\phit^{\ast}( 
		\nabla \Jinf_{t+1})
		\leq
		- D_{\phit^{\ast}}( \nabla J^T(\xst ) \mathrel{,} \nabla J^T(x_t) ) +
		D_{J^T}\left( x_{t}\mathrel{,}\xst \right)    
		- D_{J^T}\left( x_{t+1}\mathrel{,}\xst \right) 
		\;.
	\end{equation}
	Now, noting that $\phit^{\ast}( \nabla \Jinf_{t+1})$ is decreasing, we get that
	\begin{align*}
	    t \,\phit^{\ast}(\nabla \Jinf_{t})
	    &\leq \sum_{u=1}^t \phit^{\ast}(\nabla \Jinf_{u})
	    \\ &\leq 
	    \sum_{u=0}^{t-1} \left\{
	    - D_{\phit^{\ast}}( \nabla \Jinf(\xst ) \mathrel{,} \nabla \Jinf(x_u) ) +
		D_{\Jinf}\left( x_{u}\mathrel{,}\xst \right)    
		- D_{\Jinf}\left( x_{u+1}\mathrel{,}\xst \right) 
	    \right\}
	    \\ &\leq
	    \sum_{u=0}^{t-1} \left\{
	    - D_{\phit^{\ast}}( \nabla \Jinf(\xst ) \mathrel{,} \nabla \Jinf(x_{u-1}) ) 
	    \right\}
	    +
	    D_{\Jinf}\left( x_{0}\mathrel{,}\xst \right)    
		- D_{\Jinf}\left( x_{t}\mathrel{,}\xst \right) 
	    \\ &\leq
	    D_{\Jinf}\left( x_{0}\mathrel{,}\xst \right)
	    \;,
	\end{align*}
    and hence, dividing both sides by $t$, we obtain the bound in the statement of the theorem.
    
\end{proof}

\subsection{Proof of Theorem~\ref{thm:convergence-quadratic-thm}}
\begin{proof}
    Over the course of this proof, we will use the short-hand notation $\Jinf(x_t)=\Jinf_t$. We recall once more that if $\phi$ is quadratic, then $\nabla \phi$ is linear and hence we have that $D_\phi(x,y)=\phi(x-y)$. This proof follows closely the proof of Theorem~\ref{thm:convergence-basic-thm}, but where we replace the use of Lemma~\ref{lemma:dual-descent-Lemma} with Lemma~\ref{lemma:primal-descent-lemma}. We begin by noting that Lemma~\ref{lemma:primal-descent-lemma}-2 and the convexity of $\phi$ imply that $\Jinf_{t+1}\leq \Jinf_t$, and hence $\{ \Jinf_{t} \}$ is non-increasing. Next, recalling Lemma~\ref{lemma:primal-descent-lemma}-3, we have that
    \begin{equation} \label{eq:descent-proof-proximal-bound-2}
		D_{\Jinf}( x_{t+1} \mathrel{,} \xst ) \leq 
		- D_{\Jinf}( \xst \mathrel{,} x_t) +
		D_{\phi}( \xst \mathrel{,} x_t)  - D_{\phi}( \xst  \mathrel{,} x_{t+1} )
		\;.
	\end{equation}
	Since $\{ \Jinf_t \}$ is non-increasing and since $D_{\Jinf}( x_t \mathrel{,} \xst ) = \Jinf_t$ we get that by using~\eqref{eq:descent-proof-proximal-bound-2},
	\begin{align*}
	    t \, \Jinf_t
	    &\leq \sum_{u=1}^t \Jinf_u
	    \\ &\leq 
	    \sum_{u=0}^{t-1} \left\{
	    - D_{\Jinf}( \xst \mathrel{,} x_u) +
		D_{\phi}( \xst \mathrel{,} x_u)  - D_{\phi}( \xst  \mathrel{,} x_{u+1} ) 
	    \right\}
	    \\ &\leq
	    \sum_{u=0}^{t-1} \left\{
	    D_{\phi}( \xst \mathrel{,} x_u)  - D_{\phi}( \xst  \mathrel{,} x_{u+1} ) 
	    \right\}
	    \\ &=
	    D_{\phi}( \xst \mathrel{,} x_0)  - D_{\phi}( \xst  \mathrel{,} x_{t} )
	    \\ &\leq
	    D_{\phi}( \xst \mathrel{,} x_0) = \phi( \xst - x_0 )
	    \;,
	\end{align*}
    hence, dividing both sides by $t$, we obtain the first bound in the statement of the theorem.
    
    To obtain the second, we note that if $\Jinf$ is also $\mu$-relatively convex with respect to $\phi^\ast$, we have that
    \begin{equation*}
        \mu D_{\phi}( x , y )
        \leq  \, D_{\Jinf}( y , x )
        \;,
    \end{equation*}
    and hence, applying this fact along with the bound Lemma~\ref{lemma:dual-descent-Lemma}-3, we obtain that
    \begin{align*}
    	D_{\phi}( \xst \mathrel{,} x_t) 
    	&\geq 
		D_{\phi}( \xst  \mathrel{,} x_{t+1} ) +
		D_{\Jinf}( x_{t+1} \mathrel{,} \xst ) +
		D_{\Jinf}( \xst \mathrel{,} x_t) 
		\\ &\geq
		D_{\phi}( \xst  \mathrel{,} x_{t+1} ) +
		\mu D_{\phi}(  x_{t+1} \mathrel{,} \xst ) +
		\mu \, D_{\phi}( \xst \mathrel{,} x_t) 
		\;,
    \end{align*}
    Now, noting that $D_{\phi}( x,y )=\phi(x-y)$, we have that
    \begin{align*}
		\phi(x_{t+1}-\xst)       	
        &\leq
        \left( \frac{1 - \mu}{1 + \mu} \right)
        \phi(x_{t}-\xst)
        \\&\leq
        \left( 1- \frac{2 \mu}{1 + \mu} \right)
        \phi(x_{t}-\xst)
        \;,
    \end{align*}
    cascading this inequality, and noting that $\mu \, \phi(x-y) \leq D_{\Jinf}(x,y) \leq \lambda \, \phi(x-y)$, we obtain the second bound in the statement of the theorem.
\end{proof}

\subsection{Proof of Lemma~\ref{thm:p+1-rate-convergence}}
\begin{proof}
    We separate the proof into two parts, the first proving~\eqref{eq:t2-rate-f} and the second proving~\eqref{eq:exp-rate-f}.
    
    \paragraph{Proof of \eqref{eq:t2-rate-f}:} Note that $a_t = f(x_t) - \fst + \phi(\Delta x_t) \geq 0$ is non-increasing by Lemma~\ref{lem:lagrangian-decreasing} and that $\Jinf_t = \sum_{s=t+1}^\infty a_s \leq \frac{C}{t}$ for $C=\lambda\, \phi(x_0-\xst)$ by equation~\eqref{eq:Jinf-bound-rel-smth-quad-conv}. Hence,
    \begin{align*}
        t^2 \, a_{2t}
        \leq t \, \sum_{u=t+1}^{2t} a_u
        \leq
        t \, \Jinf_t
        \leq
        C
        \;.
    \end{align*}
    Therefore, we have that
    \begin{equation}
        \limsup_{t\rightarrow\infty} t^2 a_t \leq 4 \, C
        \;,
    \end{equation}
    and hence by definition of the $\limsup$, $a_t > \frac{4\,C}{t^2}$ for at most finitely many $t$, and we have the desired result.
    
    \paragraph{Proof of \eqref{eq:exp-rate-f}:}  Here, we use a version of the \emph{reverse Stolz-Cesàro} theorem. From equation~\eqref{eq:Jinf-bound-rel-conv-quad-conv}, we have that
    \begin{equation*}
        \Jinf_t = \Jinf(x_t) \leq c_0 e^{-c_1 t} = b_t
    \end{equation*}
    for $c_0 = \lambda \, \phi(x_0-\xst)$ and $c_1 = -\log( 1-\frac{2\gamma}{1+\gamma} )$. Hence, we have that
    \begin{equation*}
        \limsup_{t\rightarrow\infty} \frac{\Jinf_t}{b_t} \in [0,1]
        \;\;\text{and}\;\;
        \lim_{t\rightarrow\infty} \frac{b_{t}}{b_{t+1}} = e^{-c_1} = B \neq 1
        \;.
    \end{equation*}
    Hence, noting that 
    \begin{align*}
        \frac{\Jinf_{t}-\Jinf_{t+1}}{b_{t+1}-b_t} 
        = \frac{ \frac{\Jinf_{t}}{b_{t}} \frac{b_{t}}{b_{t+1}} - \frac{\Jinf_{t+1}}{b_{t+1}} }{1 - \frac{b_{t}}{b_{t+1}}}
    \end{align*}
    we compute
    \begin{align*}
        \limsup_{t\rightarrow\infty}\frac{f(x_{t+1}) +\phi(\Delta x_t) }{b_{t+1}-b_t}
        &=
        \limsup_{t\rightarrow\infty}\frac{\Jinf_{t}-\Jinf_{t+1}}{b_{t+1}-b_t} 
        \\&=
        \limsup_{t\rightarrow\infty}
        \frac{ \frac{\Jinf_{t}}{b_{t}} \frac{b_{t}}{b_{t+1}} - \frac{\Jinf_{t+1}}{b_{t+1}} }{1 - \frac{b_{t}}{b_{t+1}}}
        \\&\leq
        \limsup_{t\rightarrow\infty}
        \frac{ \frac{\Jinf_{t}}{b_{t}} \frac{b_{t}}{b_{t+1}} }{1 - \frac{b_{t}}{b_{t+1}}}
        \leq
        \frac{B}{1- B} 
        \;.
    \end{align*}
    Hence, we have that $\limsup_{t\rightarrow\infty} \frac{f(x_{t+1}) +\phi(\Delta x_t) }{b_{t+1}-b_t} \leq  \frac{B}{1- B} $, which by definition implies that $\frac{f(x_{t+1}) +\phi(\Delta x_t) }{b_{t+1}-b_t}> \frac{B}{1- B} $ at most finitely many times, giving the desired result.
\end{proof}

\subsection{Proof of Lemma~\ref{lem:loss-gateaux-deriv-connection}}

\begin{proof}
    Recall the definition of the Gâteaux derivative from equation~\eqref{eq:gateaux-expression},
    \begin{equation*}
        \mcR_T^\prime(\xb)(\delxb) = \sum_{t=1}^{T} \langle \nabla\phi(\Delta x_{t}) - \nabla\phi(\Delta x_{t-1}) + \nabla f(x_t) \mathrel{,} \delx_{t-1} \rangle
		\;.
    \end{equation*}
    Computing the dual norm using the above expression, we find that
    \begin{align*}
       \left\| \mcR_T^\prime(\xb^\theta) \right\|^2_{2,\ast}
       &= 
       \sum_{t=1}^{T} \left\| \nabla\phi(\Delta x_{t}) - \nabla\phi(\Delta x_{t-1}) + \nabla f(x_t) \right\|^2
       \\ &=
       \sum_{t=1}^T \L(\theta;x_{t-1})
       \;,
    \end{align*}
    as desired.
\end{proof}


\end{document}